\documentclass{article}

\usepackage{amsmath,amssymb,amsthm}
\usepackage{csquotes}
\usepackage{xspace}
\usepackage{array}
\usepackage{multirow}
\usepackage{booktabs}
\usepackage[font={small}]{caption}
\usepackage{natbib}

\newtheorem{theorem}{Theorem}[section]
\newtheorem{lemma}[theorem]{Lemma}
\newtheorem{definition}[theorem]{Definition}

\newtheorem{assumption}[theorem]{Assumption}

\DeclareMathOperator*{\argmax}{arg\,max}

\newcommand{\Y}{\mathcal{Y}}

\usepackage{microtype}
\usepackage{graphicx}
\usepackage{subfigure}
\usepackage{booktabs} 

\usepackage[accepted,nohyperref]{icml2018}

\icmltitlerunning{Interpreting Blackbox Models via Model Extraction}

\begin{document}

\twocolumn[
\icmltitle{Interpreting Blackbox Models via Model Extraction}




\begin{icmlauthorlist}
\icmlauthor{Osbert Bastani}{a}
\icmlauthor{Carolyn Kim}{b}
\icmlauthor{Hamsa Bastani}{a}
\end{icmlauthorlist}

\icmlaffiliation{a}{University of Pennsylvania}
\icmlaffiliation{b}{Stanford University}

\icmlcorrespondingauthor{Osbert Bastani}{obastani@seas.upenn.edu}


\vskip 0.3in
]



\printAffiliationsAndNotice{}  

\begin{abstract}
  Interpretability has become incredibly important as machine learning is increasingly used to inform consequential decisions. We propose to construct \emph{global explanations} of complex, blackbox models in the form of a decision tree approximating the original model---as long as the decision tree is a good approximation, then it mirrors the computation performed by the blackbox model. We devise a novel algorithm for extracting decision tree explanations that actively samples new training points to avoid overfitting. We evaluate our algorithm on a random forest to predict diabetes risk and a learned controller for cart-pole. Compared to several baselines, our decision trees are both substantially more accurate and equally or more interpretable based on a user study. Finally, we describe several insights provided by our interpretations, including a causal issue validated by a physician.
\end{abstract}

\section{Introduction}

Machine learning has revolutionized our ability to use data to inform critical decisions, such as medical diagnosis~\cite{diagnosis2001, caruana2015intelligible, valdes2016mediboost}, bail decisions for defendants~\cite{bail2017, jung2017} and the design of aircraft collision avoidance systems~\cite{collision2010}. At the same time, machine learning models have been shown to exhibit unexpected defects when deployed in the real world, such as causality issues (i.e., inability to distinguish causal effects from correlations)~\cite{pearl2009, caruana2015intelligible}, fairness (i.e., internalizing prejudices present in training data)~\cite{dwork2012fairness, fairness2016}, and covariate shift (i.e., differences in the training and test distributions)~\cite{covariateshift2000, covariateshift2009}.

Interpretability is a promising way to address these challenges~\cite{rudin2014algorithms,doshi2017towards}, since it enables data scientists to diagnose issues in machine learning models~\cite{caruana2012intelligible,wang2015falling,letham2015interpretable,slim,lime}. There are three approaches to interpretability. First, we can deploy an interpretable model such as a decision tree or a rule list in production~\cite{caruana2012intelligible,wang2015falling,letham2015interpretable,slim}, allowing the data scientist to validate the deployed model; however, this approach often requires sacrificing accuracy. Alternatively, we can use a complex model in production, but provide a \emph{local explanation} for each of its predictions~\cite{lime}; however, the data scientist must now validate each prediction made. In contrast, our approach is to extract a \emph{global explanation} in the form of an interpretable model that approximates the complex model; as long as the approximation quality is good, then the interpretable model mirrors the computation performed by the complex model. Thus, by inspecting the interpretable model, the data scientist can diagnose issues in the complex model. To maximize applicability, we treat the complex model as a \emph{blackbox}, i.e., we only require the ability to run the model on a chosen input and observe the corresponding output.

We use decision trees as global explanations, since they are nonparametric (so they can closely approximate complex models) but highly structured (so they are interpretable). The challenge is that decision trees are hard to learn~\cite{frosst2017distilling}. We propose a \emph{model extraction} algorithm for learning decision trees---to avoid overfitting, our algorithm generates new training data by actively sampling new inputs and labeling them using the complex model. We evaluate our algorithm on two benchmarks: a random forest trained to predict diabetes risk and a control policy for cart-pole~\cite{cartpole_problem}. We find that our decision trees are substantially more accurate (relative to the complex model) compared to several baselines. An important question is whether decision trees are interpretable. A key contribution of our work is a user study evaluating interpretability by asking machine learning graduate students to perform tasks such as computing counterfactuals and identifying risky subpopulations; we find that our decision trees are equally or more interpretable compared to the baselines. Finally, we describe several insights based on our interpretations, including a causal issue validated by a physician.

\paragraph{Related work.}

There has been much interest in learning interpretable models, including decision trees~\cite{CART}, rule lists~\cite{wang2015falling,letham2015interpretable}, sparse linear models~\cite{tibshirani1996regression,slim,jung2017}, generalized additive models~\cite{caruana2012intelligible}, and decision sets~\cite{lakkaraju2016interpretable}. There has been work using model compression~\cite{bucilua2006model} to learn decision trees~\cite{breiman1996born,frosst2017distilling}, but they use rejection sampling, whereas our active sampling strategy directly targets paths most in need of additional data, thereby substantially improving accuracy. There have also been approaches focused on extracting decision trees from specific model families such as random forests~\cite{van2007seeing,deng2014interpreting,vandewiele2016genesim}; in contrast, our approach is fully blackbox, enabling it to work with any model family. There has been work on constructing global explanations: \emph{relative influence} scores the contribution of each feature in random forests~\cite{friedman2001greedy}, and~\cite{datta2016algorithmic} uses the Shapley value to quantify feature influence. In recent work, \cite{lakkaraju2017interpretable} extracts global explanations in the form of decision sets; we show that our decision trees are equally interpretable while achieving higher accuracy relative to the complex model.

\section{Problem Formulation}
\label{sec:problem}

Our algorithm learns axis-aligned decision trees~\cite{CART}. An \emph{axis-aligned constraint} is a constraint $C=(x_i\le t)$, where $i\in[d]=\{1,...,d\}$ and $t\in\mathbb{R}$, where $d$ is the dimension of the input space $\mathcal{X}$. More general constraints can be built from existing constraints using negations $\neg C$, conjunctions $C_1\wedge C_2$, and disjunctions $C_1\vee C_2$. The \emph{feasible set} of $C$ is $\mathcal{F}(C)=\{x\in\mathcal{X}\mid x\text{ satisfies }C\}$.

A \emph{decision tree} $T$ is a binary tree. An \emph{internal node} $N=(N_L,N_R,C)$ of $T$ has a left child node $N_L$ and a right child node $N_R$, and is labeled with an axis-aligned constraint $C=(x_i\le t)$. A \emph{leaf node} $N=(y)$ of $T$ is associated with a label $y\in\mathcal{Y}$. We use $N_T$ to denote the root node of $T$. The decision tree is interpreted as a function $T:\mathcal{X}\to\mathcal{Y}$ in the usual way. More precisely, a leaf node $N=(y)$ is interpreted as a function $N(x)=y$, an internal node $N=(N_L,N_R,C)$ is interpreted as a function $N(x)=N_L(x)$ if $x\in\mathcal{F}(C)$, and $N(x)=N_R(x)$ otherwise. Then, $T(x)=N_T(x)$. For a node $N\in T$, we let $C_N$ denote the conjunction of the constraints along the path from the root of $T$ to $N$. More precisely, $C_N$ is defined recursively: for the root $N_T$, we have $C_{N_T}=\text{True}$, and for an internal node $N=(N_L,N_R,C)$, we have $C_{N_L}=C_N\wedge C$ and $C_{N_R}=C_N\wedge\neg C$.

Given a training set $X_{\text{train}}\subseteq\mathcal{X}$ and blackbox access to a function $f:\mathcal{X}\to\mathcal{Y}$, our goal is to learn a decision tree $T:\mathcal{X}\to\mathcal{Y}$ that approximates $f$. We focus on the case $\mathcal{X}=\mathbb{R}^d$ and $\mathcal{Y}=[m]$ (i.e., classification); our approach easily generalizes to the case where $\mathcal{X}$ contains categorical dimensions, and to $\mathcal{Y}=\mathbb{R}$ (i.e., regression). For classification, we measure performance using accuracy relative to $f$ on a held out test set, i.e., $\frac{1}{|X_{\text{test}}|}\sum_{x\in X_{\text{test}}}\mathbb{I}[T(x)=f(x)]$. For binary classification, we use $F_1$ score, and for regression, we use mean-squared error.

\begin{table*}
\small
\centering
\resizebox{\textwidth}{!}{
\begin{tabular}{llrlrrll}
\toprule
\multicolumn{1}{c}{{\bf Dataset}} &
\multicolumn{1}{c}{{\bf Task}} &
\multicolumn{1}{c}{{\bf \# Features}} &
\multicolumn{1}{c}{{\bf Outcomes}} &
\multicolumn{1}{c}{{\bf \# Training}} &
\multicolumn{1}{c}{{\bf \# Test}} &
\multicolumn{1}{c}{{\bf Blackbox Model}} &
\multicolumn{1}{c}{{\bf Blackbox Performance}} \\
\hline
diabetes risk & classification & 384 & $\{\text{high risk},~\text{low risk}\}$ & 404 & 174 & random forest & $F_1=$~0.24  \\
cart-pole~\cite{cartpole_problem} & reinforcement learning & 4 & $\{\text{left},~\text{right}\}$ & 100 & 100 & control policy & $\text{reward}=$~200.0 \\
\bottomrule
\end{tabular}
}
\caption{Summary of the datasets used in our evaluation.}
\label{tab:summary}
\end{table*}

\section{Decision Tree Extraction Algorithm}
\label{sec:algorithm}

Our algorithm first uses $X_{\text{train}}$ to estimate a distribution $\mathcal{P}$ over $\mathcal{X}$. For scalability, our algorithm greedily constructs the decision tree $T$: it initializes $T$ to a single leaf (the root), and then iteratively splits leaf nodes in $T$. To split a leaf $N\in T$, it uses an active sampling strategy to sample new inputs $x\sim\mathcal{P}$ such that $x\in\mathcal{F}(C_N)$, computes the corresponding labels $y=f(x)$, and uses this data to identify the best split. We first describe the \emph{exact greedy decision tree} $T^*$, i.e., the decision tree extracted using infinite data, and then describe how our algorithm estimates $T^*$.

\paragraph{Input distribution.}

We fit a mixture of axis-aligned Gaussians to $X_{\text{train}}$ using EM. The categorical distribution over mixtures is $j\sim\text{Categorical}(\phi)$ (where $\phi\in\mathbb{R}^k$), and the mixture distributions are $x\sim\mathcal{N}(\mu_j,\Sigma_j)$ for each $j\in[K]$ (where $\mu\in\mathbb{R}^{Kd}$ and $\Sigma\in\mathbb{R}^{Kd^2}$, and each $\Sigma_j$ is diagonal).


\paragraph{Exact greedy decision tree.}

We construct the \emph{exact greedy decision tree} $T^*$ of size $k$ similar to CART~\cite{CART}. We initialize $T^*$ to a single leaf $N_{T^*}=(y)$, where $y$ is the majority label according to $\mathcal{P}$. Then, we iteratively split leaves in $T^*$ (using a total of $k-1$ iterations)---at each iteration, we choose a leaf $N=(y)$ in $T^*$ and replace it with an internal node $N'=(N_L,N_R,C)$, where $N_L=(y_L)$ and $N_R=(y_R)$ are new leaf nodes, and $C=(x_{i^*}\le t^*)$, where
\begin{align}
\label{eqn:exbranch}
(i^*,t^*)=\argmax_{i\in[d],t\in\mathbb{R}}G(i,t),
\end{align}
where the gain
\begin{align}
\label{eqn:gain}
G(i,t)=&~-H(f,C_N\wedge(x_i\le t))\\
&~-H(f,C_N\wedge(x_i>t)) + H(f,C_N) \nonumber \\
H(f,C)=&\bigg(1-\sum_{y\in\mathcal{Y}}\text{Pr}_{x\sim\mathcal{P}}[f(x)=y\mid C]^2\bigg)\cdot\text{Pr}_{x\sim\mathcal{P}}[C] \nonumber
\end{align}
uses the (weighted) Gini impurity $H$ (other metrics can be used as well). The leaf node labels are
\begin{align}
\label{eqn:exlabel}
y_L&=\argmax_{y\in\mathcal{Y}}\text{Pr}_{x\sim\mathcal{P}}[f(x)=y\mid C_N\wedge(x_i\le t)] \\
y_R&=\argmax_{y\in\mathcal{Y}}\text{Pr}_{x\sim\mathcal{P}}[f(x)=y\mid C_N\wedge(x_i>t)]. \nonumber
\end{align}
We choose to replace the leaf $N\in T^*$ with the highest gain (\ref{eqn:gain}); we terminate early if its gain is zero.


\paragraph{Estimated greedy decision tree.}

Given $n\in\mathbb{N}$, our algorithm constructs a greedy decision tree $\hat{T}$ in the same way as the construction of $T^*$, except (\ref{eqn:gain}) and (\ref{eqn:exlabel}) are estimated using $n$ i.i.d. samples $x\sim\mathcal{P}\mid C_N$ each (we select which leaf $N\in T^*$ to expand using $n$ additional samples). We describe how to sample $x\sim\mathcal{P}\mid C$, where $C$ is a conjunction of axis-aligned constraints
\begin{align*}
C=&(x_{i_1}\le t_1)\wedge...\wedge(x_{i_k}\le t_k) \\
&\wedge(x_{j_1}>s_1)\wedge...\wedge(x_{j_h}>s_h).
\end{align*}
Constraints in $C$ may be redundant---(i) for two constraints $x_i\le t$ and $x_i\le t'$ such that $t\le t'$, the first constraint implies the second, so we can discard the latter, and (ii) for two constraints $x_i>s$ and $x_i>s'$ such that $s\ge s'$, we can similarly discard the latter. Given two constraints $x_i\le t$ and $x_i>s$, we can assume that $t\ge s$ (otherwise $C$ is unsatisfiable, so the gain (\ref{eqn:gain}) is zero and the algorithm terminates). In summary, we can assume $C$ contains at most one inequality $(x_i\le t)$ and at most one inequality $(x_i>s)$ for each $i\in[d]$, and if both are present, then the two are not mutually exclusive. For simplicity, we assume $C$ contains both inequalities for each $i\in[d]$:
\begin{align*}
C=(s_1\le x_1\le t_1)\wedge...\wedge(s_d\le x_d\le t_d).
\end{align*}
Now, recall that $\mathcal{P}$ is a mixture of axis-aligned Gaussians, so it has probability density function
\begin{align*}
p_{\mathcal{P}}(x)&=\sum_{j=1}^K\phi_j\cdot p_{\mathcal{N}(\mu_j,\Sigma_j)}(x) \\
&=\sum_{j=1}^K\phi_j\prod_{i=1}^dp_{\mathcal{N}(\mu_{ji},\sigma_{ji})}(x_i),
\end{align*}
where $\sigma_{ji}=(\Sigma_j)_{ii}$. The conditional distribution is
\begin{align*}
p_{\mathcal{P}\mid C}(x)&\propto\sum_{j=1}^K\phi_j\prod_{i=1}^dp_{\mathcal{N}(\mu_{ji},\sigma_{ji})\mid C}(x_i) \\
&=\sum_{j=1}^K\phi_j\prod_{i=1}^dp_{\mathcal{N}(\mu_{ji},\sigma_{ji})\mid(s_i\le x_i\le t_i)}(x_i).
\end{align*}
Since the Gaussians are axis-aligned, the unnormalized probability of each component is
\begin{align*}
\tilde{\phi}_j'&=\int\phi_j\prod_{i=1}^dp_{\mathcal{N}(\mu_{ji},\sigma_{ji})\mid(s_i\le x_i\le t_i)}(x_i)dx \\
&=\phi_j\prod_{i=1}^d\left(\Phi\left(\frac{t_i-\mu_{ji}}{\sigma_{ji}}\right)-\Phi\left(\frac{s_i-\mu_{ji}}{\sigma_{ji}}\right)\right),
\end{align*}
where $\Phi$ is the cumulative density function of the standard Gaussian distribution $\mathcal{N}(0,1)$. Then, the normalization constant is $Z=\sum_{j=1}^K\tilde{\phi}_j'$, and the component probabilities are $\tilde{\phi}=Z^{-1}\tilde{\phi}'$. Finally, to sample $x\sim\mathcal{P}\mid C$, we sample $j\sim\text{Categorical}(\tilde{\phi})$, and $x_i\sim\mathcal{N}(\mu_{ji},\sigma_{ji})\mid(s_i\le x_i\le t_i)$ (for each $i\in[d]$). We use standard algorithms for sampling truncated Gaussian distributions to sample each $x_i$.
  
\paragraph{Theoretical guarantees.}

We show that $\hat{T}\to T^*$ as $n\to\infty$. A related result is~\cite{domingos2000mining}, but their analysis is limited to discrete features, for which convergence is much easier to analyze. We give proofs in Appendix~\ref{sec:mainresultsappendix}.
\begin{assumption}
\label{assump:p}
The density $p(x)$ of $\mathcal{P}$ is continuous, bounded ($p(x)\le p_{\text{max}}$), and has bounded domain ($p(x)=0$ for $|x|>x_{\text{max}}$).
\end{assumption}
To satisfy this assumption, we can truncate our Gaussian mixture model; this modification does not affect $T^*$ or $\hat{T}$ very much since Gaussians have exponential tails.
\begin{assumption}
\label{assump:unique}
The maximizers $(i^*,t^*)$ in (\ref{eqn:exbranch}), and $y_L$ and $y_R$ in (\ref{eqn:exlabel}) are unique.
\end{assumption}
In other words, there are no nodes where the Gini impurity for two different choices of branch are exactly tied (such a tie is very unlikely in practice); this assumption ensures that $T^*$ is well defined. We now define the notion in which the extracted tree converges to the exact tree. For simplicity, we additionally assume that our trees are complete (i.e., have all nodes up to a given depth $D$).
\begin{definition}
\label{def:exact}
Let $T,T'$ be decision trees. For any $\epsilon>0$, we say $T$ is an $\epsilon$-approximation of $T'$ if $\text{Pr}_{x\sim\mathcal{P}}[T(x)=T'(x)]\ge1-\epsilon$. For any $\epsilon,\delta>0$, we say $T$ is $(\epsilon,\delta)$-exact if $\text{Pr}[T\text{ is an }\epsilon\text{ approximation of }T^*]\ge1-\delta$ (probability over the training samples $x\sim\mathcal{P}$).
\end{definition}
\begin{theorem}
\label{thm:main}
For all $\epsilon,\delta>0$, $\exists n\in\mathbb{N}$ such that $\hat{T}$ extracted using $n$ samples is $(\epsilon,\delta)$-exact.
\end{theorem}

\section{Evaluation}
\label{sec:evaluation}

\begin{figure*}
\centering
\begin{tabular}{cc}
\includegraphics[width=0.45\textwidth]{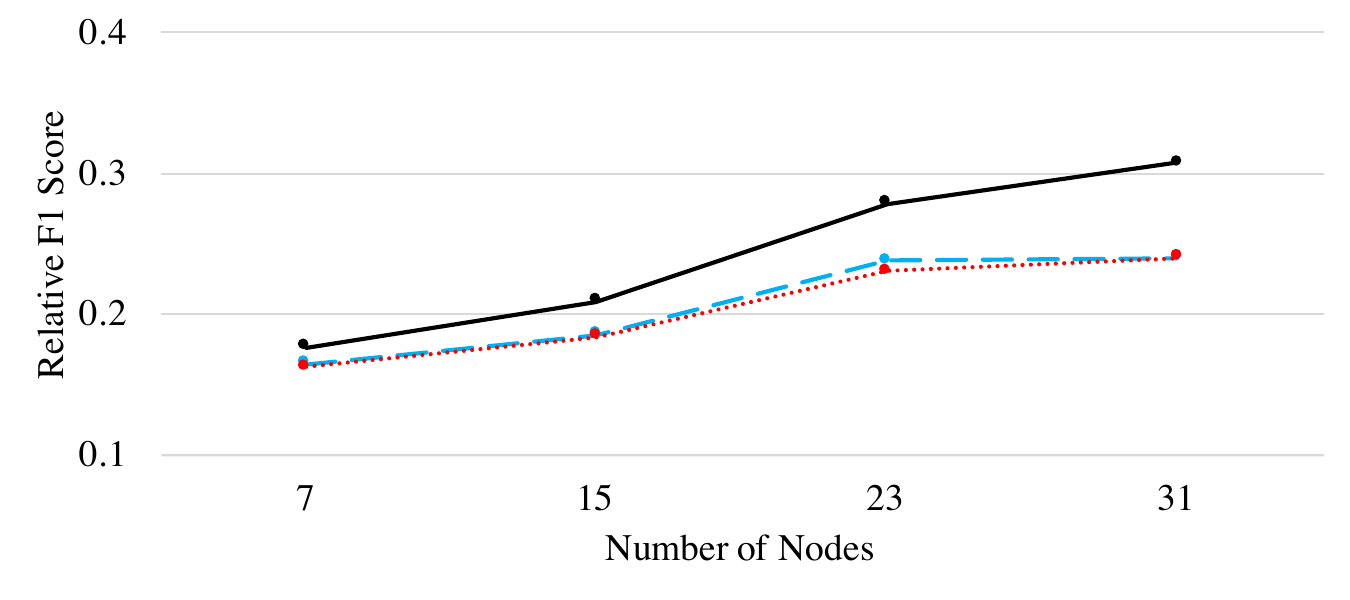} &
\includegraphics[width=0.45\textwidth]{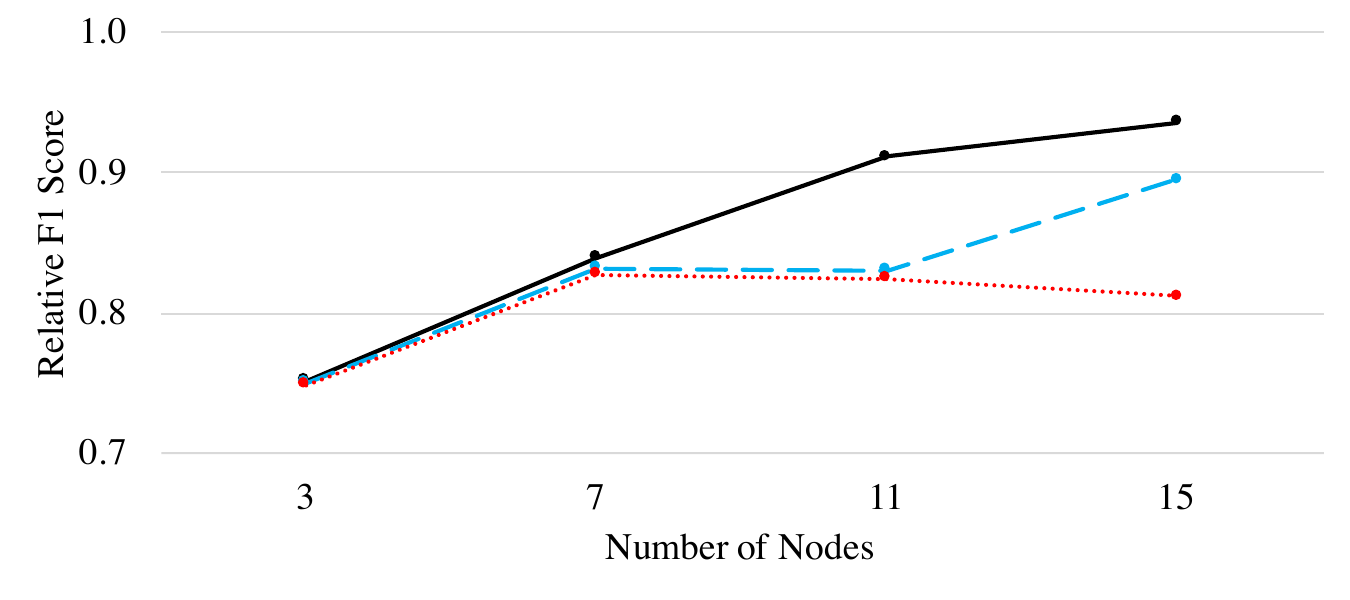} \\
(a) & (b) \\
\includegraphics[width=0.45\textwidth]{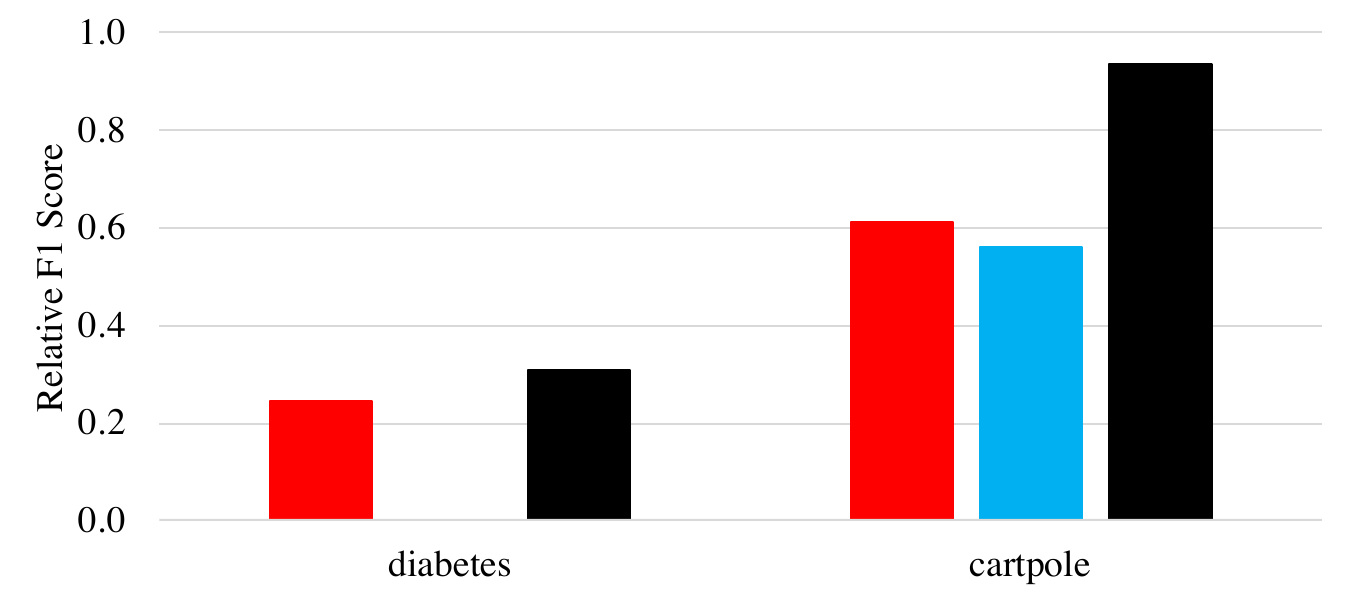} &
\includegraphics[width=0.45\textwidth]{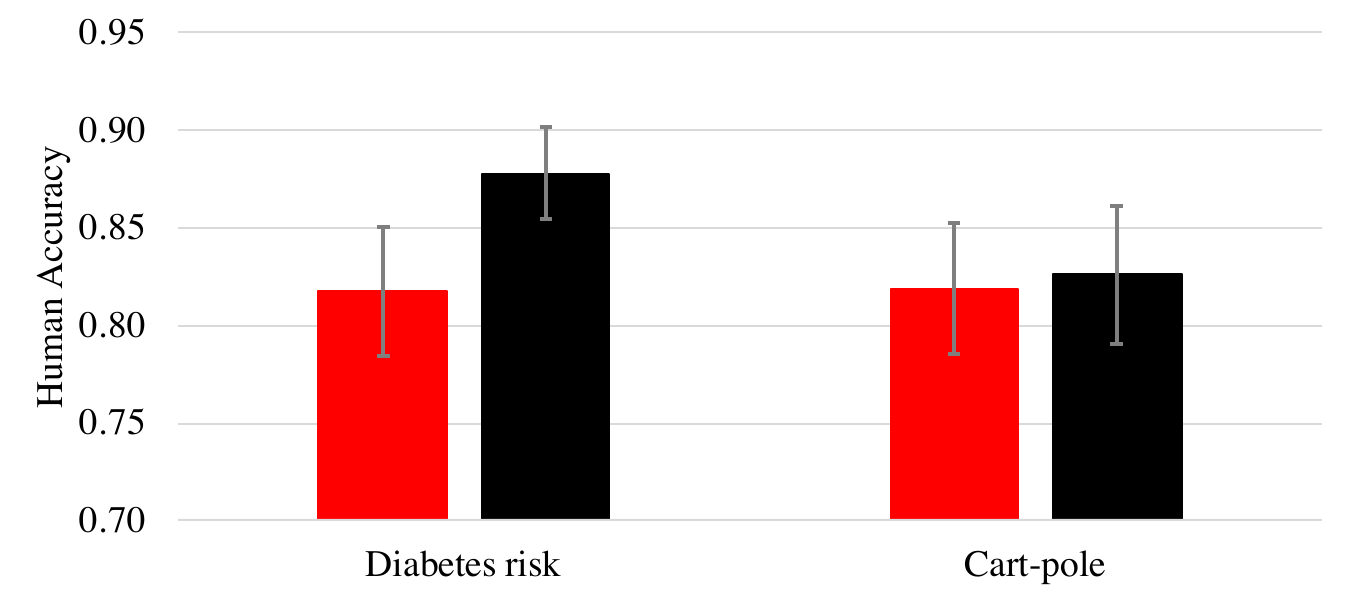} \\
(c) & (d)
\end{tabular}
\caption{Fidelity on (a) the diabetes risk benchmark, and (b) the cart-pole benchmark, of decision trees learned using CART (red, dotted), the born-again algorithm (blue, dashed), and our algorithm (black, solid). (c) Fidelity of rule lists (red), decision sets (blue), and our decision trees (black); the decision set learning algorithm did not scale to the diabetes risk benchmark. (d) User response accuracy for the baseline rule list or decision set (red) and our decision trees (black).}
\label{fig:fidelity}
\end{figure*}

We compare the \emph{fidelity} (i.e., accuracy relative to the complex model) and interpretability of our decision trees to several baselines on two benchmarks (see in Table~\ref{tab:summary}). We also compare the fidelity of our algorithm to CART, a state-of-the-art algorithm, on a number of additional benchmarks.

\paragraph{Diabetes risk.}

The goal of this dataset is to predict whether a patient has high or low risk for type II diabetes, based on their ICD-9 diagnosis codes, prescribed medications, and demographics. In part of our evaluation, we compare models trained on patients from different healthcare providers. Thus, we initially focus on the largest provider (578 patients); later, for comparison, we use another large provider (402 patients). We balance the training set (only 11.8\% of patients have diabetes) and train a random forest to predict risk. We extract decision tree explanations using $n=1000$.

\paragraph{Cart-pole.}

The goal of the cart-pole problem~\cite{cartpole_problem} is to balance a pole on top of a cart. We discretize the state space, estimate the transition probabilities and rewards using random samples, and use value iteration to compute the optimal policy. We extract decision tree explanations using $n=200$ (fewer samples are needed since the dimension of $\mathcal{X}$ is much smaller).

\begin{table*}
\begin{scriptsize}
\begin{center}
\begin{tabular}{llrrlrrr}
\toprule
\multicolumn{1}{c}{{\bf Dataset}} &
\multicolumn{1}{c}{{\bf Task}} &
\multicolumn{1}{c}{{\bf Samples}} &
\multicolumn{1}{c}{{\bf Features}} &
\multicolumn{1}{c}{{\bf Model}} &
\multicolumn{1}{c}{{\bf Score of $f$}} &
\multicolumn{1}{c}{{\bf Fidelity (Ours)}} &
\multicolumn{1}{c}{{\bf Fidelity (CART)}} \\
\hline
breast cancer~\cite{wolberg1990multisurface} & classify & 569 & 32 & forest & 0.966 $F_1$ & {\bf 0.957 $F_1$} & 0.945 $F_1$ \\
breast cancer & classify & 569 & 32 & neural net & 0.951 $F_1$ & {\bf 0.956 $F_1$} & 0.949 $F_1$ \\
dermatology~\cite{guvenir1998learning} & classify & 366 & 34 & forest & 0.994 $F_1$ & {\bf 0.970 $F_1$} & 0.967 $F_1$ \\
dermatology & classify & 366 & 34 & neural net & 0.988 $F_1$ & {\bf 0.997 $F_1$} & 0.964 $F_1$ \\
prostate cancer~\cite{stamey1989prostate} & classify & 97 & 9 & forest & 0.749 $F_1$ & {\bf 0.900 $F_1$} & 0.818 $F_1$ \\
prostate cancer & classify & 97 & 9 & neural net & 0.723 $F_1$ & {\bf 0.842 $F_1$} & 0.820 $F_1$ \\
auto mpg~\cite{quinlan1993combining} & regress & 398 & 8 & forest & 8.62 MSE & {\bf 2.29 MSE} & 2.51 MSE \\
auto mpg & regress & 398 & 8 & neural net & 13.77 MSE & {\bf 2.37 MSE} & 2.59 MSE \\
student grade~\cite{cortez2008using} & regress & 382 & 33 & forest & 4.47 MSE & {\bf 0.40 MSE} & 0.64 MSE \\
student grade & regress & 382 & 33 & neural net & 6.60 MSE & {\bf 4.27 MSE} & 5.10 MSE \\
mountain car~\cite{mountaincar_problem} & reinforce & 100 & 2 & control policy & $R=-$140.0 & {\bf 81.3\%} & 78.6\% \\
pendulum~\cite{openai_pendulum} & reinforce & 100 & 3 & control policy & $R=-$638.2 & {\bf 0.56 MSE} & 1.66 MSE \\
\bottomrule
\end{tabular}
\caption{Comparison of our algorithm to CART. For (binary) classification, fidelity is $F_1$ score on the test set, and for regression, it is MSE. For reinforcement learning, fidelity is accuracy (discrete actions) or MSE (continuous actions) on the test set. On every problem instance, our algorithm outperforms CART in terms of fidelity.}
\label{tab:results}
\end{center}
\end{scriptsize}
\end{table*}

\subsection{Fidelity}

High fidelity ensures that the extracted decision tree reflects the blackbox model. We measure fidelity using $F_1$ score on the held-out test set $\tilde{X}_{\text{test}}=\{(x,f(x))\mid x\in X_{\text{test}}\}$, where $X_{\text{test}}$ is the original test set. All results are medians over 20 random train/test splits. We compare to CART trees~\cite{CART} and born-again trees~\cite{breiman1996born}, rule lists~\cite{letham2015interpretable,yang2017scalable}, and decision sets~\cite{lakkaraju2016interpretable,lakkaraju2017interpretable}. The born-again algorithm requires an input distribution; we use our estimated Gaussian mixture model $\mathcal{P}$.

\paragraph{Comparison to other decision trees.}

In Figure~\ref{fig:fidelity} (a) and (b), we compare the fidelity of our decision trees to that of CART trees and born-again trees for varying sizes (i.e., total number of nodes $k$); we outperform both baselines in every case. For larger decision trees, our active sampling strategy greatly reduces overfitting. In contrast, the born-again algorithm is unable to generate a substantial number of new training points at deeper levels of the tree since it uses rejection sampling.

\paragraph{Comparison to other model families.}

In Figure~\ref{fig:fidelity} (c), we compare the fidelity of our decision trees (size 31) to rule lists (trained using~\cite{yang2017scalable}) and decision sets (trained using~\cite{lakkaraju2016interpretable}). We use implementations obtained from the authors. Both require us to bin continuous features; for diabetes risk, we bin age into 7 bins, and for cart-pole, we use the discretization in our MDP. We find that~\cite{lakkaraju2016interpretable} does not scale to the diabetes risk benchmark, likely because it has hundreds of features. Our decision tree substantially outperforms the baselines. The difference on cart-pole is especially large; we believe the difference arises since features are binned---thus, the decision tree inputs are 4 dimensional, whereas the rule list and decision set inputs are 28 dimensional, making them more likely to overfit.

\paragraph{Stability of decision trees.}

Theorem~\ref{thm:main} suggests that our decision trees should become more and more stable as the number of samples $n\to\infty$. We examine whether this behavior holds empirically for decision trees extracted for the diabetes risk benchmark. In particular, we used both our algorithm and the born-again algorithm to extract 10 random decision trees $T_1,...,T_{10}$. For each pair $T_k$ and $T_h$ (where $k\neq h$), we count the fraction of corresponding nodes in $T_k$ and $T_h$ that are equal (i.e., the branches at those nodes have equal values of $i$ and $t$). Then, we computed the average over possible pairs. Intuitively, this metric captures the similarity between two random extracted decision trees. This metric is 0.52 (for our algorithm) vs. 0.22 (for the born-again algorithm) when $n=2000$, and 0.67 (for our algorithm) vs. 0.25 (for the born-again algorithm) for $n=20000$. Thus, our algorithm produces substantially more stable trees than the born-again algorithm (since it is able to get more samples to estimate each branch).

\paragraph{Additional comparisons to CART.}

We compare the fidelity of our algorithm to CART on a range of datasets; see Figure~\ref{tab:results} for results. We outperform CART in every problem instance.

\subsection{Interpretability for Diabetes Risk}

\begin{figure}
\centering
\begin{tabular}{c}
\includegraphics[width=0.45\textwidth]{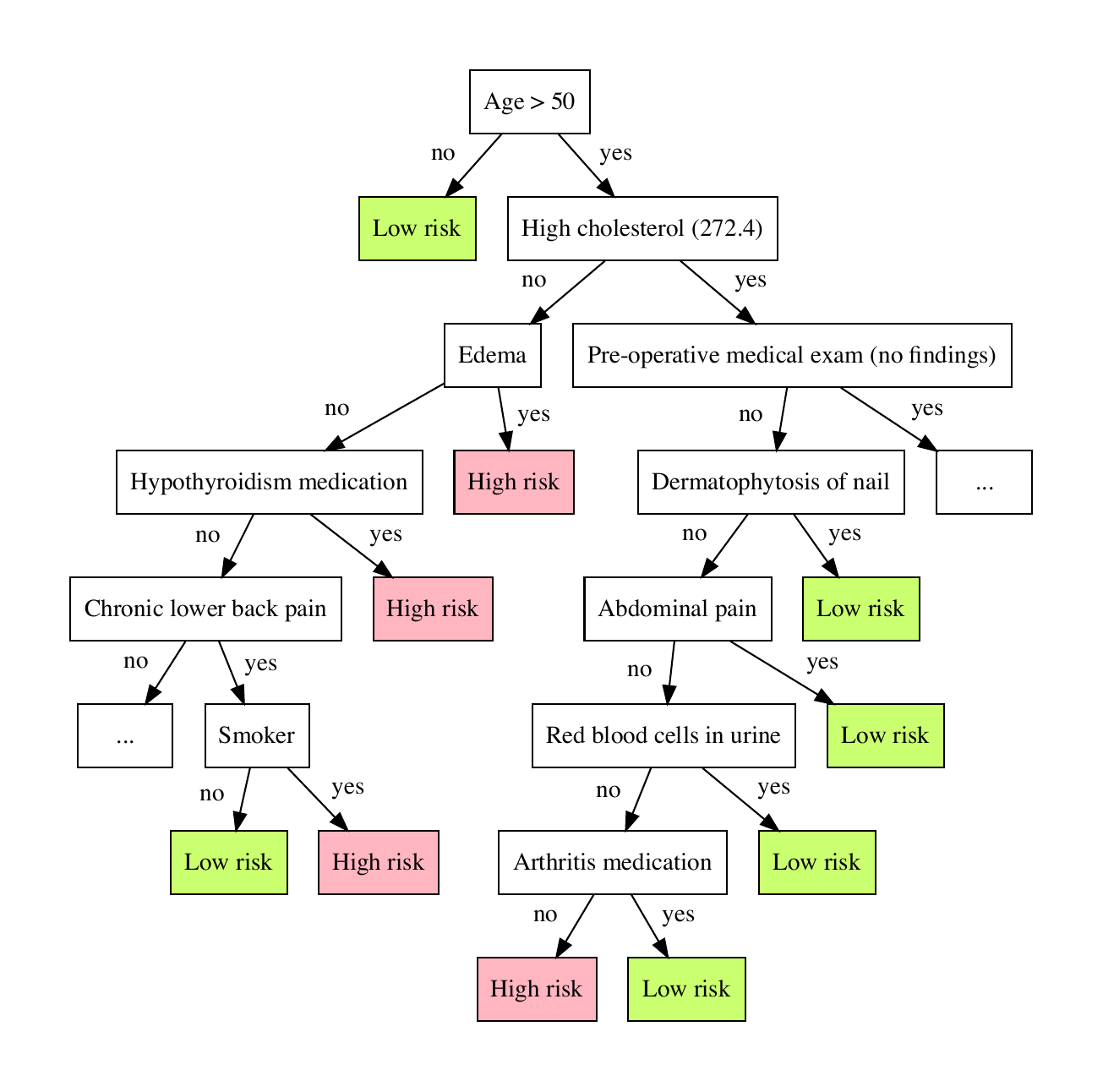} \\
(a) \\\\
\includegraphics[width=0.45\textwidth]{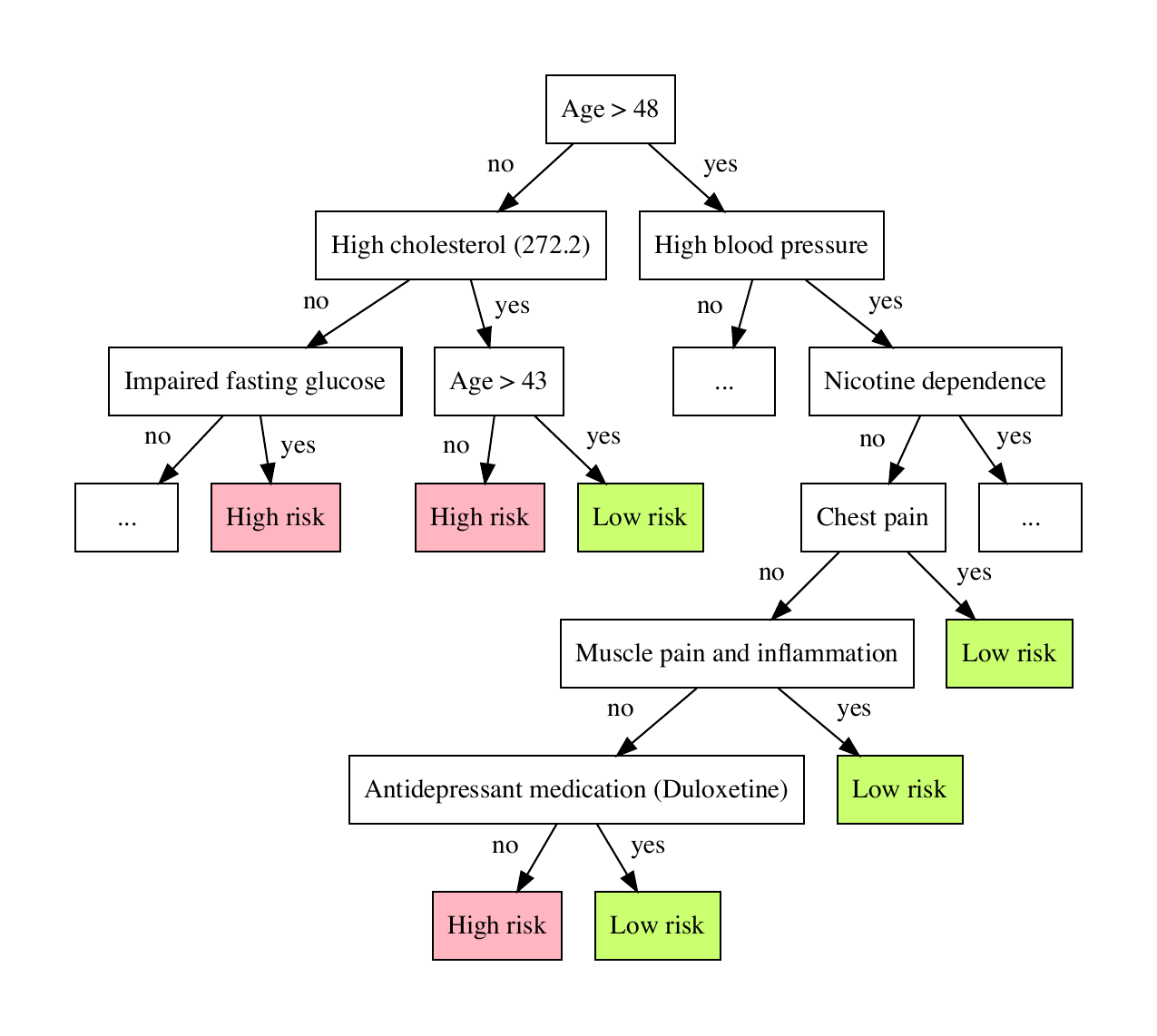} \\
(b) \\\\
\includegraphics[width=0.45\textwidth]{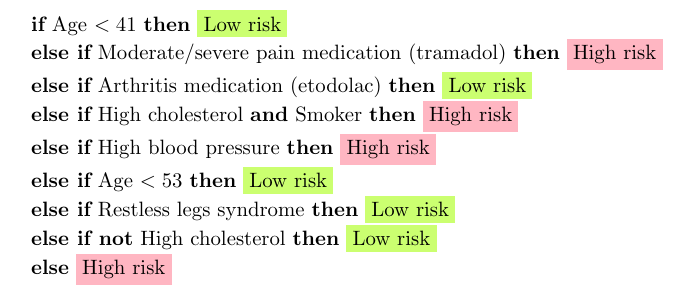} \\
(c)
\end{tabular}
\caption{Global explanations of a risk random forest for predicting diabetes risk: (a) our decision tree, (b) our decision tree for a random forest trained on data from an alternate provider, and (c) a rule list.}
\label{fig:diabetesmodels}
\end{figure}

We performed a user study to evaluate the interpretability of our decision trees. Since our goal is to enable data scientists familiar with machine learning to understand and validate the complex model, we recruited 46 graduate students with a background in machine learning for our study. Each participant answered questions intended to test their understanding of various interpretations; we asked them to skip a question if they were unable to determine the answer in 1-2 minutes. We show images of our user study interface in Appendix~\ref{sec:userstudyappendix}. We randomized the order of the models and corresponding questions. First, we describe the results for the diabetes risk benchmark.

\paragraph{Interpretations.}

We compare the two interpretations of the random forest: (i) a (randomly chosen) decision tree with 31 nodes extracted by our algorithm; a simplified version is shown in Figure~\ref{fig:diabetesmodels} (a), and (ii) a rule list extracted using~\cite{yang2017scalable}, shown in Figure~\ref{fig:diabetesmodels} (c). We visualize decision trees using Graphviz, and rule lists as if-then-else programs, which is the formatting used in prior work~\cite{yang2017scalable} (all our users were familiar with programming). Otherwise, we tried to keep the visualizations consistent.

\newcolumntype{P}{ >{\arraybackslash} p{0.45\textwidth} }
\newcolumntype{M}{ >{\centering\arraybackslash} m{0.45\textwidth} }

\begin{figure*}
\centering
\begin{tiny}
\begin{tabular}{PP}
Consider patients over 50 years old who are otherwise healthy and are not taking any medications. According to the decision tree, are these patients at a high risk for diabetes?
\begin{itemize}
\item Yes 
\item No
\end{itemize}
&
Consider patients over 53 years old who are otherwise healthy and are not taking any medications. According to the rule list, are these patients at a high risk for diabetes?
\begin{itemize}
\item Yes 
\item No 
\end{itemize}
\\
Smoking is known to increase risk of diabetes, so the local hospital has started a program to help smokers quit smoking. According to the decision tree, which patient subpopulation should we target in this program if we want to reduce diabetes risk?
\begin{itemize}
\item Patients over 50 years old who have high cholesterol
\item Patients over 50 years old who have chronic lower back pain
\item Patients over 50 years old who have high cholesterol, edema, chronic lower back pain, and who take medication for hypothyroidism
\end{itemize}
&
Smoking is known to increase risk of diabetes, so the local hospital has started a program to help smokers quit smoking. According to the rule list, which patient subpopulation should we target in this program if we want to reduce diabetes risk?
\begin{itemize}
\item Patients over 41 years old 
\item Patients over 41 years old who have high cholesterol 
\item Patients over 41 years old who have high cholesterol, and take medication for arthritis
\end{itemize}
\end{tabular}
\end{tiny}
\caption{Examples of questions asked in our user study on the diabetes risk benchmark, for our decision tree (left) and for the rule list (right).}
\label{fig:diabetesquestions}
\end{figure*}

\paragraph{Questions.}

We designed five questions to test whether the user could understand an interpretable model. To ensure fairness, these questions were designed independently of the interpretations. However, we needed to adapt the questions to each of the two interpretations---in particular, we modified the possible answers to fit the structure of the interpretation so there was a single correct answer. Two examples of questions are shown in Figure~\ref{fig:diabetesquestions}; the variants on the left are for the decision tree, and those on the right are for the rule list. The first question tests whether the user can determine how the interpretable model classifies a given patient. The second question tests whether the user is able to identify the subpopulation for which ``Smoker'' is a relevant feature; enabling users to understand these subpopulation-level effects is a useful benefit of global explanations. 

\paragraph{Results.}

We show user accuracies for each interpretation in Figure~\ref{fig:fidelity} (d) (averaged across users and questions). Users responded equally or more accurately for the decision tree, even though it is much larger than the rule list; this effect is significant ($p=0.02$ using a paired $t$-test with 46 samples). For every question, a majority of users answered correctly, so we believe our questions were fair.

\paragraph{Difficulty with conditional structure.}

We found that users had difficulty understanding the conditional structure of the rule list---one of our questions required users to determine that only the first three lines of the model were relevant for patients taking arthritis medication, but only 65\% of users answered correctly. For our decision tree, 91\% of users correctly answered the corresponding question. This effect has also been identified in previous work~\cite{lakkaraju2016interpretable}.

\paragraph{Interpretation vs. blackbox model.}

The goal of our user study is to measure how well users can understand the interpretation. Intuitively, having high fidelity should ensure that the answers to questions according to the interpretation accurately reflect the true answers according to the blackbox model. To be sure, we evaluated how often these two answers are equal. In particular, each question in our evaluation (except question 3 on the diabetes dataset) asks for either a prediction (e.g., ``Are patients age $\ge41$ classified as high risk?'') or a counterfactual (e.g., ``For a smoker age $\ge41$ who is a smoker, does an intervention that gets them to quit smoking reduce diabetes risk?''). For a prediction, we compute the true answer according to the blackbox model using the formula
\begin{align*}
y=\arg\max_{y\in\Y}\text{Pr}_{x\sim\mathcal{P}}[f(x)=y\mid C],
\end{align*}
where $f$ is the blackbox model and $C$ is the condition on $x$ (e.g., $\text{age}\ge41$). Similarly, for a counterfactual, we use the formula
\begin{align*}
\mathbb{E}_{x\sim\mathcal{P}}[f(x)\mid C,\text{do}(C')]-\mathbb{E}_{x\sim\mathcal{P}}[f(x)\mid C]\stackrel{?}{\le}0,
\end{align*}
where $C$ is the condition (e.g., $\text{age}\ge41\wedge\text{smoker}=\text{true}$) and $C'$ is the counterfactual (e.g., $\text{do}(\text{smoker}=\text{false})$). Of the 4 questions regarding the diabetes dataset (since we omit question 3), all 4 of the answers according to the decision tree matched the true answer, whereas only 3 of the answers according to the rule list matched.

\subsection{Discussion of the Diabetes Risk Classifier}

\paragraph{Variations across providers.}

We can use interpretations to understand differences in random forests trained on patients from different providers. In Figure~\ref{fig:diabetesmodels} (b), we show a decision tree trained on data from an alternate provider, which contained EMRs for 402 patients. We can immediately identify differences in how diagnoses were reported. For example, there are several ICD-9 codes corresponding to high cholesterol; for the original provider, 30\% of patients were diagnosed with 272.4 (``unspecified hyperlipidemia''), whereas only 2\% of patients were diagnosed with 272.2 (``mixed hyperlipidemia''). In contrast, for the alternate provider, 19\% of patients were diagnosed with 272.2, and 18\% of patients were diagnosed with 272.4. As another example, for the alternate provider, 10\% of patients were diagnosed with ``Impaired fasting glucose'', which appears in Figure~\ref{fig:diabetesmodels} (c). In contrast, for the original provider, only 1\% of patients have this diagnosis; indeed, this feature never shows up in interpretations for random forests trained on patients from the original provider. Understanding such covariate shifts among the patient populations for different providers can help data scientists adapt existing models to new providers.

\paragraph{Dependence on previous doctor visits.}

A notable feature of the decision tree in Figure~\ref{fig:diabetesmodels} (a) is the subtree rooted at the node labeled ``Dermatophytosis of nail''. In this subtree, if the patient has any of the diagnoses listed, then the decision tree classifies the patient as low risk; otherwise, they are classified as high risk. This effect occurs across providers, e.g., in subtree rooted at ``Chest pain'' in Figure~\ref{fig:diabetesmodels} (b). This effect is present (and statistically significant) in the data and in the random forest.

This effect is likely non-causal---in an interview with a physician, we learned that these diagnoses have no known relationship with diabetes risk. After examining the decision tree, the physician suggested a plausible explanation---patients who have these diagnoses are more likely to have visited a doctor at least once in the past year prior to their diabetes diagnosis, upon which the doctor may have recommended interventions designed to reduce the patient's diabetes risk. In contrast, patients who have not visited a doctor in the past year may not have realized they were at high risk for diabetes, especially since this subtree is conditioned on patients who are over 50 years old and have high cholesterol, which are both known risk factors for diabetes. Indeed, we found that among patients over 50 years old with high cholesterol, diabetes risk was actually (statistically significantly) higher if the patient had zero previous doctor visits than if the patient had a single previous visit (despite the fact that we would expect patients with one previous visit to be sicker).

Understanding such non-causal effects is important~\cite{caruana2015intelligible}---many patients in the subpopulation defined by this subtree are in particular need of preventative interventions, yet the classifier is proposing to discontinue interventions precisely for these patients---and our interpretations provide a promising way to do so. We note that relative influence scores cannot identify this effect, since they do not examine patient subpopulations, and the effect only applies to the subpopulation of patients that are at least 50 years old and have high cholesterol. For example, the correlation of ``Abdominal pain'' with high risk is $8.1\times10^{-3}$; however, conditioned on age greater than 50 and having high cholesterol, the correlation is $-9.8\times10^{-2}$. Indeed, none of the features in this subtree appear in the top 40 relative influence scores of the random forest.

\paragraph{Non-monotone dependence on age.}

Note that age appears twice in the decision tree in Figure~\ref{fig:diabetesmodels} (c). Typically, younger patients are at lower risk for diabetes; however, conditioned on being less than 48 years old and having high cholesterol, the classifier predicts higher risk for younger patients. While we cannot be certain of the cause, there are a number of possible explanations. For example, it may be the case that a diagnosis of high cholesterol in younger patients is abnormal and therefore much more indicative of high diabetes risk. Alternatively, doctors may be more likely to urge older patients with high cholesterol to take preventative measures to reduce diabetes risk.

This structure demonstrates how the decision tree can capture non-monotone dependencies on continuous features such as age. In contrast, non-monotone dependencies cannot be captured by relative influence scores. Rule lists can capture such dependencies, but their structure makes it more difficult to understand the effect---for example, to reason about the relationship between the first and sixth rules of the rule list in Figure~\ref{fig:diabetesmodels} (c), we also have to reason about the four intermediate rules.

\subsection{Interpretability for Cart-pole}

\begin{figure}
\centering
\begin{tabular}{M}
\includegraphics[width=0.35\textwidth]{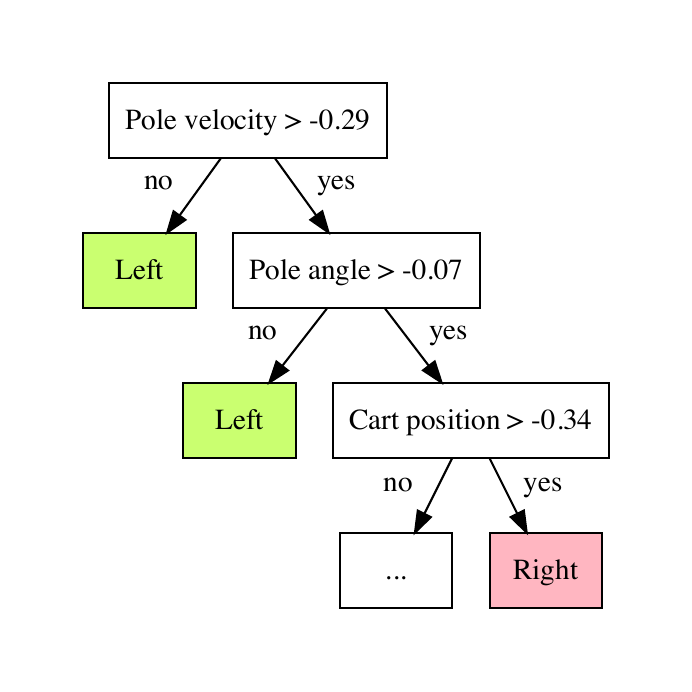} \\
(a) \\
\includegraphics[width=0.35\textwidth]{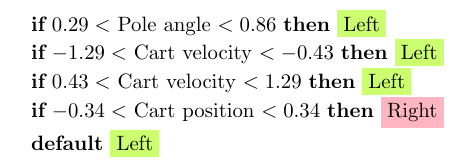} \\
(b)
\end{tabular}
\caption{Global explanations of the cart-pole policy: (a) our decision tree, and (b) a decision set.}
\label{fig:cartpolemodels}
\end{figure}

Next, we describe the part of our user study focused on the cart-pole control policy.

\paragraph{Interpretations.}

We compare the two interpretations of the control policy: (i) a decision tree with 15 nodes extracted using our algorithm; a simplified version is shown in Figure~\ref{fig:cartpolemodels} (a), and (ii) a decision set trained using~\cite{lakkaraju2016interpretable}, shown in Figure~\ref{fig:cartpolemodels} (b).

\paragraph{Questions.}

We designed three questions to check whether user could understand the interpretations. Two questions test whether the user can reason about desirable symmetries of the policy, e.g.:
\begin{displayquote}
In theory, the action taken should not depend on the position of the cart. Does the decision tree satisfy this property?
\end{displayquote}
The third question tests whether the user can compute how the interpretation acts in given state. As for the diabetes risk benchmark, we adapt each question to each of the two interpretations.

\paragraph{Results.}

We show user accuracies for each interpretation in Figure~\ref{fig:fidelity} (d) (averaged across users and questions). Users responded equally or more accurately for our decision tree, even though it is larger (the difference is not statistically significant). For every question, a majority of users answered correctly, so we believe our questions were fair.

\paragraph{Interpretation vs. blackbox model.}

As with the questions about the diabetes dataset, we checked how often the answers according to the interpretation matched the true answers according to the blackbox model. In this case, for both the decision tree and the decision set, all of the answers matched.

\subsection{Discussion of the Cart-Pole Controller}

\paragraph{Translation invariance.}

We expect that the motion of the cart-pole should be invariant to translating the cart position. However, it is easy to see from both the decision tree and the decision set in Figure~\ref{fig:cartpolemodels} that the learned policy does not exhibit this symmetry. This asymmetry likely arises because the MDP simulation always starts from a similar initial position. Thus, the cart position is highly correlated with its velocity, so the control policy can use the two interchangeably to predict what action to take. Understanding this bias in the control policy is important because it indicates that the control policy may not generalize well if the initial position changes substantially.

\paragraph{Reflection invariance.}

We also expect the motion of the cart-pole to be invariant to reflection across the $y$-axis (i.e., flip left and right). However, the models in Figure~\ref{fig:cartpolemodels} do not exhibit this symmetry. This asymmetry likely arises because in the MDP simulation, the pole typically initially falls toward the left. Thus, to maximize performance, the control policy focuses on stopping the pole from falling toward the left, which requires moving the cart toward the left. As before, the control policy may not generalize well if we change the initial direction in which the pole is falling.

\section{Conclusion}

We have proposed an approach for interpreting blackbox models based on decision tree extraction, and shown how it can be used to interpret blackbox models. Important directions for future work include devising algorithms for model extraction using more expressive input distributions, and developing new ways to gain insight from the extracted decision trees.


\bibliographystyle{icml2018}
\bibliography{paper}

\clearpage

\onecolumn

\appendix
\section{Proofs of Main Results}
\label{sec:mainresultsappendix}

In this section, we give a proof of Theorem~\ref{thm:main}.

\subsection{Proof Overview}

At a high level, the idea behind our proof of Theorem~\ref{thm:main} is to show that the internal structure of $T$ converges to that of $T^*$. Intuitively, this result holds because as we estimate $T$ a larger and larger number of samples, the parameters $(i,t)$ of each internal node of $T$ and the parameters $(y)$ of each leaf node of $T$ should converge to the parameters of $T^*$. As long as the internal node parameters converge, then an input $x\in\mathcal{X}$ should be routed to leaf nodes in $T$ and $T^*$ at the same position. Then, as long as the internal node parameters converge, then $x$ should furthermore be assigned the same label by $T$ and $T^*$.

The key challenge is that the internal node parameter $t$ is continuous, so $T$ always has some error compared to $T^*$. Thus, to prove Theorem~\ref{thm:main}, we have to quantify this error and show that it goes to zero as $n$ goes to infinity. Intuitively, we quantify this error as the probability that an input is routed to the wrong leaf node in $T$.

Our main lemma formalizes this notion. We begin by establishing some notation. Consider a node $N^*$ in the exact greedy decision tree $T^*$. We define the function $\phi:T^*\to T$ to map $N^*$ to the node $N=\phi(N^*)$ at the corresponding position in the estimated greedy decision tree $T$ estimated using $n$ samples. Now, given an input $x\in\mathcal{X}$, we write $x\xrightarrow{T^*}N^*$ if $x$ is routed to node $N^*$ in $T^*$, and similarly $x\xrightarrow{T}N$ if $x$ is routed to node $N$ in $T$. Finally, we denote the leaves of $T^*$ and $T$ by $\textsf{leaves}(T^*)$ and $\textsf{leaves}(T)$, respectively.

Then, we have the following key result:
\begin{lemma}
\label{lem:mainbody}
Let $p(x)$ be the probability density function for the distribution $\mathcal{P}$, let $N^*\in T^*$ and $N=\phi(N^*)$, and let
\begin{align*}
p_{N^*}(x)&=p(x)\cdot\mathbb{I}[x\xrightarrow{T^*}N^*] \\
p_N(x)&=p(x)\cdot\mathbb{I}[x\xrightarrow{T}N].
\end{align*}
Then, $\|p_N-p_{N^*}\|_1$ converges in probability to $0$ (where the randomness is taken over the $n$ samples used to extract $T$), i.e., for any $\epsilon,\delta>0$, there exists $n>0$ such that
\begin{align*}
\|p_N-p_{N^*}\|_1\le\epsilon
\end{align*}
with probability at least $1-\delta$.
\end{lemma}
Intuitively, $p_{N^*}$ captures the distribution of points that are routed to $N^*$ in $T^*$, and $p_N$ captures the distribution of points that are routed to $N$ in $T$. Then, this lemma says that the distribution of points routed to $N^*$ and $N$ are similar. We prove this lemma in Section~\ref{sec:mainbodyproof}.

\subsection{Proof of Main Theorem}

We now use Lemma~\ref{lem:mainbody} to prove Theorem~\ref{thm:main}. In particular, we must show that the quantity $P=\text{Pr}_{x\sim\mathcal{P}}[T(x)\neq T^*(x)]$ is bounded by $\epsilon$ with probability at least $1-\delta$. Throughout the proof, we use Lemma~\ref{lem:mainbody} with parameters $\left(\frac{\epsilon}{K},\frac{\delta}{2K}\right)$, i.e., we have $\|p_N-p_{N^*}\|_1\le\frac{\epsilon}{K}$ with probability at least $1-\frac{\delta}{2K}$. By a union bound, this fact holds for every leaf node in $T^*$ with probability at least $1-\frac{\delta}{2}$.

Then, our proof proceeds in two steps:
\begin{enumerate}
\item We show that a leaf node $N\in T$ is correctly labeled as long as $\epsilon$ is sufficiently small. More precisely, let $N^*\in\textsf{leaves}(T^*)$ such that $N^*=(y^*)$, and let $N=\phi(N^*)\in\textsf{leaves}(T)$ such that $N=(y)$; then, we show that for any $\delta'>0$, there exists $n\in\mathbb{N}$ such that $y=y^*$ with probability at least $1-\delta'$ (where the randomness is taken over the $n$ samples used to extract $T$).
\item Using the Lemma~\ref{lem:mainbody} together with the first step, we show that $P\le\epsilon$ with probability at least $1-\delta$.
\end{enumerate}

\paragraph{Proving $y=y^*$.}

Let $p(x)$ be the probability density function for the distribution $\mathcal{P}$, and let $N^*\in\textsf{leaves}(T^*)$ such that $N^*=(y^*)$ and $N=\phi(N^*)\in\textsf{leaves}(T)$ such that $N=(y)$. First, we rewrite the objective (\ref{eqn:exlabel}) in terms of $p_{N^*}$. In particular, for each $y'\in\mathcal{Y}$, let
\begin{align*}
p_{y'}^*&=\text{Pr}_{x\sim\mathcal{P}}[f(x)=y'\wedge(x\xrightarrow{T^*}N^*)] \\
&=\int\mathbb{I}[f(x)=y']\cdot\mathbb{I}[x\xrightarrow{T^*}N^*]\cdot p(x)dx.
\end{align*}
Then, we have $y^*=\argmax_{y'\in\mathcal{Y}}p_{y'}^*$, since the denominator $\text{Pr}_{x\sim\mathcal{P}}[x\xrightarrow{T^*}N^*]$ in (\ref{eqn:exlabel}) is constant with respect to $y'$. Similarly, we rewrite the objective (\ref{eqn:exlabel}) in terms of $p_{N^*}$, letting
\begin{align*}
p_{y'}=\frac{1}{n}\sum_{j=1}^n\mathbb{I}[f(x^{(j)})=y']\cdot\mathbb{I}[x^{(j)}\xrightarrow{T}N]
\end{align*}
for each $y'\in\mathcal{Y}$, in which case we have $y=\argmax_{y'\in\mathcal{Y}}p_{y'}$.

By Assumption~\ref{assump:unique}, we know that $y^*$ is the unique maximizer of $p_y^*$, i.e.,
\begin{align*}
\Delta=p_{y^*}^*-\argmax_{y'\neq y^*}p_{y'}^*>0.
\end{align*}
Therefore, to show that $y=y^*$, it suffices to show that for each $y'\in\mathcal{Y}$, we have
\begin{align*}
|p_{y'}-p_{y'}^*|\le\frac{\Delta}{3},
\end{align*}
since then, for each $y'\in\mathcal{Y}$, we have
\begin{align*}
p_{y^*}-p_{y'}\ge\left(p_{y^*}^*-\frac{\Delta}{3}\right)-\left(p_{y'}^*+\frac{\Delta}{3}\right)\ge\frac{\Delta}{3}>0,
\end{align*}
which implies that $y=y^*$ since $y^*$ is the maximizer of $p_{y'}$.

To show that $|p_{y'}-p_{y'}^*|\le\Delta/3$, we first define
\begin{align*}
\tilde{p}_{y'}=\int\mathbb{I}[f(x)=y']\cdot\mathbb{I}[x\xrightarrow{T}N]\cdot p(x)dx.
\end{align*}
Then, we have
\begin{align*}
|p_{y'}-p_{y'}^*|\le|p_{y'}-\tilde{p}_{y'}|+|\tilde{p}_{y'}-p_{y'}^*|.
\end{align*}
To bound the first term, let
\begin{align*}
d_{y'}=\mathbb{I}[f(x)=y']\cdot\mathbb{I}[x\xrightarrow{T}N]
\end{align*}
be a Bernoulli random variable, so
\begin{align*}
d_{y'}^{(j)}=\mathbb{I}[f(x^{(j)})=y']\cdot\mathbb{I}[x^{(j)}\xrightarrow{T}N]
\end{align*}
are samples of $d_{y'}$ for $j\in[n]$. Then, we have $\tilde{p}_{y'}=\mathbb{E}[d_{y'}]$ and $p_{y'}=n^{-1}\sum_{j=1}^nd_{y'}^{(j)}$, so we can apply Hoeffding's inequality to get
\begin{align*}
\text{Pr}\left[|p_{y'}-\tilde{p}_{y'}|>\frac{\Delta}{6}\right]\le2\exp\left(-\frac{n\Delta^2}{18}\right).
\end{align*}
To bound the second term, note that
\begin{align*}
|\tilde{p}_{y'}-p_{y'}^*|&=\left|\int\mathbb{I}[f(x)=y']\cdot(\mathbb{I}[x\xrightarrow{T}N]-\mathbb{I}[x\xrightarrow{T^*}N^*])\cdot p(x)dx\right| \\
&\le\int|\mathbb{I}[x\xrightarrow{T}N]-\mathbb{I}[x\xrightarrow{T^*}N^*]|\cdot p(x)dx \\
&=\|p_N-p_{N^*}\|_1 \\
&\le\epsilon.
\end{align*}

Finally, assume that $\epsilon<\Delta/6$; then, taking a union bound over $y'\in\mathcal{Y}$, we have that
\begin{align*}
|p_{y'}-p_{y'}^*|\le\frac{\Delta}{3}
\end{align*}
for all $y'\in\mathcal{Y}$ with probability at least $1-\delta'$, where
\begin{align*}
\delta'=2\cdot|\mathcal{Y}|\cdot\exp\left(-\frac{n\Delta^2}{18}\right).
\end{align*}
In particular, it follows that $y=y^*$ with probability at least $1-\delta'$.

\paragraph{Bounding $P$.}

First, we separate the contribution of each leaf node to $P$:
\begin{align*}
P
&=\text{Pr}_{x\sim\mathcal{P}}[T(x)\neq T^*(x)] \\
&=\sum_{N^*\in\textsf{leaves}(T^*)}\text{Pr}_{x\sim\mathcal{P}}[T(x)\neq T^*(x)\text{ and }x\xrightarrow{T^*}N^*].
\end{align*}

Next, we apply the result from the first step of this proof with parameter $\delta'=\frac{\delta}{2K}$ (where $K$ is the number of nodes in each $T^*$ and $T$); then, for any leaf node $N^*\in\textsf{leaves}(T^*)$, the label assigned to $N^*$ equals the label assigned to $N$ with probability at least $1-\frac{\delta}{2K}$. Taking a union bound over the leaf nodes, this fact holds true for all the leaf nodes with probability at least $1-\frac{\delta}{2}$. For the remainder of the proof, we assume that this fact holds.

Consider an input $x$ such that $x\xrightarrow{T^*}N^*$; as long as $N^*$ and $\phi(N^*)$ have the same label, and additionally $x\xrightarrow{T}\phi(N^*)$, then $T(x)=T^*(x)$. Thus, we have
\begin{align*}
\text{Pr}_{x\sim\mathcal{P}}[T(x)\neq T^*(x)\text{ and }x\xrightarrow{T^*}N^*]\le\text{Pr}_{x\sim\mathcal{P}}[\neg(x\xrightarrow{T}\phi(N^*))\text{ and }x\xrightarrow{T^*}N^*].
\end{align*}
As a consequence, we have
\begin{align*}
P&\le\sum_{N^*\in\textsf{leaves}(T^*)}\text{Pr}_{x\sim\mathcal{P}}[\neg(x\xrightarrow{T}\phi(N^*))\text{ and }x\xrightarrow{T^*}N^*] \\
&=\sum_{N^*\in\textsf{leaves}(T^*)}\int(1-\mathbb{I}[x\xrightarrow{T}\phi(N^*)])\cdot\mathbb{I}[x\xrightarrow{T^*}N^*]\cdot p(x)dx.
\end{align*}

Now, we claim that
\begin{align*}
(1-\mathbb{I}[x\xrightarrow{T}\phi(N^*)])\cdot\mathbb{I}[x\xrightarrow{T^*}N^*]\le|\mathbb{I}[x\xrightarrow{T^*}N^*]-\mathbb{I}[x\xrightarrow{T}\phi(N^*)]|.
\end{align*}
To see this claim, note that both sides of inequality take values in $\{0,1\}$. Furthermore, the right-hand side equals $0$ only if the two indicators are equal. In this case, the left-hand side also equals $0$, so the claim follows. Thus, we have
\begin{align*}
P&\le\sum_{N^*\in\textsf{leaves}(T^*)}\int|\mathbb{I}[x\xrightarrow{T^*}N^*]-\mathbb{I}[x\xrightarrow{T}\phi(N^*)]|\cdot p(x)dx \\
&=\sum_{N^*\in\textsf{leaves}(T^*)}\int|\mathbb{I}[x\xrightarrow{T^*}N^*]\cdot p(x)-\mathbb{I}[x\xrightarrow{T}\phi(N^*)]\cdot p(x)|dx \\
&=\sum_{N^*\in\textsf{leaves}(T^*)}\|p_{\phi(N^*)}-p_{N^*}\|_1.
\end{align*}

By Lemma~\ref{lem:mainbody}, we have
\begin{align*}
P&\le\sum_{N^*\in\textsf{leaves}(T^*)}\frac{\epsilon}{K}\le\epsilon.
\end{align*}

Since Lemma~\ref{lem:mainbody} holds with probability at least $1-\frac{\delta}{2}$, and the first part of this proof holds with probability at least $1-\frac{\delta}{2}$, by a union bound, we have $P\le\epsilon$ with probability at least $1-\delta$, which completes the proof.

\subsection{Proof of Main Lemma}
\label{sec:mainbodyproof}

The key idea behind proving Lemma~\ref{lem:mainbody} is to use induction on the structure of the tree. More precisely, it is clear that Lemma~\ref{lem:mainbody} holds for the root node $N_{T^*}$ of $T^*$, since every input is routed to the root, i.e.,
\begin{align*}
\mathbb{I}[x\xrightarrow{T^*}N_{T^*}]=\mathbb{I}[x\xrightarrow{T}\phi(N_{T^*})]=1
\end{align*}
for all $x\in\mathcal{X}$. Then, it suffices to show that if Lemma~\ref{lem:mainbody} holds for the parent of a node $N^*\in T^*$, then it holds for $N^*$ as well.

More precisely, let $M^*$ be the parent of $N^*$, and let $M=\phi(M^*)$ be the parent of $N=\phi(N^*)$. Our goal is to prove that, assuming
\begin{align*}
  \|p_M-p_{M^*}\|_1\xrightarrow{p}0,
\end{align*}
then
\begin{align*}
\|p_N-p_{N^*}\|_1\xrightarrow{p}0
\end{align*}
as well (note that we use $\xrightarrow{p}$ to denote convergence in probability).

For simplicity, we prove the one-dimensional case, i.e., $\mathcal{X}=\mathbb{R}$. Proving the general case is a straightforward extension of our proof, but requires extra bookkeeping that obscures the key ideas. In particular, let $N^*\in T^*$ have form $N^*=(i^*,t^*)$, and let $N=\phi(N^*)\in T$ have form $N=(i,t)$. When $d=1$, we know that $i=i^*=1$, so we only have to prove that $t$ converges to $t^*$. Proving that $i$ converges to $i^*$ is straightforward since there are only finitely many choices for $i$. With this restriction, we can assume that internal nodes have only a single parameter, i.e., $N^*=(t^*)$ where $t^*\in\mathbb{R}$, and $N=\phi(N^*)=(t)$ where $t\in\mathbb{R}$.

We begin our proof by expressing $p_N$ in terms of $p_M$. We assume without loss of generality that $N$ is the left child of $M$. Then, note that
\begin{align*}
\mathbb{I}[x\xrightarrow{T}N]&=\mathbb{I}[x\xrightarrow{T}M]\cdot\mathbb{I}[x\le t],
\end{align*}
where $M=(t)$, so we have
\begin{align*}
p_N(x)=p_M(x)\cdot\mathbb{I}[x\le t].
\end{align*}
Now, our proof proceeds in two steps:
\begin{enumerate}
\item First, we show assuming $\|p_M-p_{M^*}\|_1\xrightarrow{p}0$, then $t\xrightarrow{p}t^*$.
\item Second, we show that assuming $t\xrightarrow{p}t^*$, then $\|p_N-p_{N^*}\|_1\xrightarrow{p}0$.
\end{enumerate}

\paragraph{Step 1: Proving $t\xrightarrow{p}t^*$.}

First, we show that $\|p_M-p_{M^*}\|_1\xrightarrow{p}0$ implies
\begin{align*}
\|G-G^*\|_{\infty}\xrightarrow{p}0,
\end{align*}
where
\begin{align*}
G^*(s)&=G(i,s;\mathcal{P}\mid C_{M^*}) \\
G(s)&=G(i,s;\mathcal{P}_M)
\end{align*}
are the gain functions for $T^*$ and $T$, respectively, where $G(i,s;\mathcal{Q})$ is defined in (\ref{eqn:gain}); as noted above, we have assumed $i=1$ is a constant to simplify our exposition. Proving this step depends on the gain function being used to train the decision tree; we show that it holds for the gain function based on the Gini impurity in Lemma~\ref{lem:gini}.

Next, we show that as long as $\|G-G^*\|_{\infty}$ is sufficiently small, then the difference between their corresponding maximizers
\begin{align*}
t^*&=\argmax_sG^*(s) \\
t&=\argmax_sG(s)
\end{align*}
is small as well, i.e., $t\xrightarrow{p}t^*$.

By Assumption~\ref{assump:unique}, we can prove the existence of a \emph{gap}, which intuitively is an interval around $t^*$ outside of which the $G^*(s)$ is ``sufficiently smaller'' than $G^*(t^*)$. More precisely:
\begin{definition}
\label{def:gap}
We say that a function $g:\mathbb{R}\to\mathbb{R}$ is \emph{$(\epsilon,\delta)$-gapped} if it has a unique maximizer $s^*=\argmax_{s\in\mathbb{R}}g(s)$, and for every $s\in\mathbb{R}$ such that $|s-s^*|>\epsilon$, we have $g(s^*)>g(s)+\delta$.
\end{definition}

We show that as long as $G^*$ is continuous and has bounded support, then for any $\epsilon'>0$, there exists $\delta'>0$ such that $G^*$ is $(\epsilon',\delta')$-gapped; in Lemma~\ref{lem:ginigap}, we show that the gain function $G^*$ based on the Gini impurity satisfies these technical assumptions.  Then, let $s_{\text{max}}$ be a bound on the support of $G^*$, i.e., $G^*(s)=0$ if $|s|>s_{\text{max}}$. Let $\epsilon'>0$ be arbitrary, and let
\begin{align*}
A_{\epsilon'}=\{s\in\mathbb{R}\mid|s|\le s_{\text{max}}\text{ and }|s-s^*|\ge\epsilon'\}.
\end{align*}
Note that $A_{\epsilon'}$ is a compact set, so $G^*$ achieves its maximum on $A_{\epsilon'}$, i.e.,
\begin{align*}
t_{\epsilon'}^*=\argmax_{s\in A_{\epsilon'}}G^*(s).
\end{align*}
Then, $G^*$ is $(\epsilon',\delta')$-gapped, where
\begin{align*}
\delta'=\frac{G^*(t^*)-G^*(t_{\epsilon'}^*)}{2}>0.
\end{align*}
Note that we divide by $2$ since the inequality in Definition~\ref{def:gap} is strict.

Now, we show that having a gap implies $t\xrightarrow{p}t^*$. In particular, suppose that $\|G^*-G\|_{\infty}\le\frac{\delta'}{2}$. Then, we have
\begin{align*}
G^*(t^*)-G^*(t)&\le\left(G(t^*)+\frac{\delta'}{2}\right)-\left(G(t)-\frac{\delta'}{2}\right) \\
&\le G(t^*)-G(t)+\delta' \\
&\le\delta',
\end{align*}
where the last step follows since $t$ is the maximizer of $G$. In particular, we have shown that $|G^*(t^*)-G^*(t)|\le\delta'$, so since $G^*$ is $(\epsilon',\delta')$-gapped, it follows that $|t-t^*|\le\epsilon'$. Since $\|G^*-G\|_{\infty}\xrightarrow{p}0$, it follows that $t\xrightarrow{p}t^*$.

\paragraph{Step 2: Proving $\|p_N-p_{N^*}\|_1\xrightarrow{p}0$.}

Note that
\begin{align*}
\|p_N-p_{N^*}\|_1
&=\int|p_N(x)-p_{N^*}(x)|dx \\
&=\int|p_M(x)\cdot\mathbb{I}[x\le t]-p_{M^*}(x)\cdot\mathbb{I}[x\le t^*]|dx \\
&=\int|p_M(x)\cdot\mathbb{I}[x\le t]-(p_M(x)+p_{M^*}(x)-p_M(x))\cdot\mathbb{I}[x\le t^*]|dx \\
&\le\int p_M(x)\cdot|\mathbb{I}[x\le t]-\mathbb{I}[x\le t^*]|dx+\int|p_M(x)-p_{M^*}(x)|\cdot\mathbb{I}[x\le t^*]dx.
\end{align*}
Assume without loss of generality that $t\le t^*$. Then, for the first integral, note that the integrand equals $0$ for $x\not\in[t,t^*]$ and equals $1$ for $x\in[t,t^*]$. Thus,
\begin{align*}
\int p_M(x)\cdot|\mathbb{I}[x\le t]-\mathbb{I}[x\le t^*]|dx&=\int p_M(x)\cdot\mathbb{I}[t\le x\le t^*]dx \\
&=\int_t^{t^*}p_M(x)dx \\
&\le|t-t^*|\cdot p_{\text{max}},
\end{align*}
where the last step follows by Assumption~\ref{assump:p}, which says that $p(x)\le p_{\text{max}}$ for all $x\in\mathbb{R}$.

Next, for the second integral, note that
\begin{align*}
\int|p_M(x)-p_{M^*}(x)|\cdot\mathbb{I}[x\le t^*]dx&\le\|p_M-p_{M^*}\|_1.
\end{align*}
Together, we have
\begin{align*}
\|p_N-p_{N^*}\|_1\le\|p_M-p_{M^*}\|_1+|t-t^*|\cdot p_{\text{max}}.
\end{align*}
Since the left-hand side converges in probability to $0$, so does the right-hand side, as claimed.

\section{Proof of Technical Lemmas}

In this section, we prove the technical lemmas required for our proofs of Lemma~\ref{lem:mainbody} and Theorem~\ref{thm:main}.

\subsection{Proof of Convergence of the Gain Function}

In this section, we prove that the gain function $G$ converges uniformly to $G^*$ as $n\to\infty$. To simplify notation, we use slightly different notation for the Gini impurity $H$ compared to the definition in (\ref{eqn:gain}).

\begin{lemma}
\label{lem:gini}
Let
\begin{align*}
G^*(t)&=-H^*(f,C_{N^*}\wedge(x\le t))-H^*(f,C_{N^*}\wedge(x>t))+H^*(f,C_{N^*}) \\
H^*(f,C)&=\left(1-\sum_{y\in\mathcal{Y}}\left(\frac{\text{Pr}_{x\sim\mathcal{P}}[f(x)=y\wedge C]}{\text{Pr}_{x\sim\mathcal{P}}[C]}\right)^2\right)\cdot\text{Pr}_{x\sim\mathcal{P}}[C]
\end{align*}
be the gain function based on the Gini impurity for the exact greedy decision tree, and let
\begin{align*}
G(t)&=-H(f,C_N\wedge(x\le t))-H(f,C_N\wedge(x>t))+H(f,C_N) \\
H(f,C)&=\left(1-\sum_{y\in\mathcal{Y}}\left(\frac{\frac{1}{n}\sum_{j=1}^n\mathbb{I}[f(x^{(j)})=y\wedge x^{(j)}\in\mathcal{F}(C)]}{\frac{1}{n}\sum_{j=1}^n\mathbb{I}[x^{(j)}\in\mathcal{F}(C)]}\right)^2\right)\cdot\frac{1}{n}\sum_{j=1}^n\mathbb{I}[x^{(j)}\in\mathcal{F}(C)]
\end{align*}
be the corresponding gain function for the estimated greedy decision tree.

If $\|p_N-p_{N^*}\|_1\xrightarrow{p}0$, where
\begin{align*}
p_{N^*}(x)&=p(x)\cdot\mathbb{I}[x\xrightarrow{T^*}N^*] \\
p_N(x)&=p(x)\cdot\mathbb{I}[x\xrightarrow{T}N],
\end{align*}
then we have $\|G-G^*\|_{\infty}\xrightarrow{p}0$.
\end{lemma}
\begin{proof}
First, note that
\begin{align*}
\|G-G^*\|_{\infty}\le&\sup_{t\in\mathbb{R}}|H^*(f,C_{N^*}\wedge(x\le t))-H(f,C_N\wedge(x\le t))| \\
&+\sup_{t\in\mathbb{R}}|H^*(f,C_{N^*}\wedge(x>t))-H(f,C_N\wedge(x>t))| \\
&+\sup_{t\in\mathbb{R}}|H^*(f,C_{N^*})-H(f,C_N)|.
\end{align*}
We prove that the first term converges in probability to $0$ as $n\to\infty$; the remaining two terms can be bounded using the same argument. In particular, let
\begin{align*}
H^*(t)&=H^*(f,C_{N^*}\wedge(x\le t)) \\
H(t)&=H(f,C_N\wedge(x\le t)),
\end{align*}
so our goal is to show that $\|H-H^*\|_{\infty}\xrightarrow{p}0$. To simplify our expressions, define
\begin{align*}
g^*(t)&=\text{Pr}_{x\sim\mathcal{P}}[x\in\mathcal{F}(C_{N^*}\wedge(x\le t))] \\
h_y^*(t)&=\text{Pr}_{x\sim\mathcal{P}}[f(x)=y\wedge x\in\mathcal{F}(C_{N^*}\wedge(x\le t))] \\
g(t)&=\frac{1}{n}\sum_{j=1}^n\mathbb{I}[x^{(j)}\in\mathcal{F}(C_N\wedge(x\le t))] \\
h_y(t)&=\frac{1}{n}\sum_{j=1}^n\mathbb{I}[f(x^{(j)})=y\wedge x^{(j)}\in\mathcal{F}(C_N\wedge(x^{(j)}\le t))],
\end{align*}
A useful fact is that
\[0\le h_y^*(t)\le g^*(t)\le1\]
\[0\le h_y(t)\le g(t)\le 1\]
for all $t\in\mathbb{R}$ and all $y\in\mathcal{Y}$ (but assuming the random samples $x^{(j)}$ are fixed). Now, we have
\begin{align*}
H^*(t)&=\left(1-\sum_{y\in\mathcal{Y}}\left(\frac{h_y^*(t)}{g^*(t)}\right)^2\right)\cdot g^*(t) \\
&=g^*(t)-\sum_{y\in\mathcal{Y}}\frac{h_y^*(t)^2}{g^*(t)},
\end{align*}
and similarly
\begin{align*}
H(t)=g(t)-\sum_{y\in\mathcal{Y}}\frac{h_y(t)^2}{g(t)}.
\end{align*}
Then, we have
\begin{align*}
\|H-H^*\|_{\infty}\le\sup_{t\in\mathbb{R}}|g(t)-g^*(t)|+\sum_{y\in\mathcal{Y}}\sup_{t\in\mathbb{R}}\left|\frac{h_y^*(t)^2}{g^*(t)}-\frac{h_y(t)^2}{g(t)}\right|.
\end{align*}
We show that for a fixed $y\in\mathcal{Y}$, we have
\begin{align}
\label{eqn:single}
\sup_{t\in\mathbb{R}}\left|\frac{h_y^*(t)^2}{g^*(t)}-\frac{h_y(t)^2}{g(t)}\right|\xrightarrow{p}0.
\end{align}
Bounding the first term of $\|H-H^*\|_{\infty}$ follows similarly; together, these limits imply that $\|H-H^*\|_{\infty}\xrightarrow{p}0$ as well. We break the remainder of the proof into two steps:
\begin{enumerate}
\item First, we prove that $\|g-g^*\|_{\infty}\xrightarrow{p}0$ and $\|h_y-h_y^*\|_{\infty}\xrightarrow{p}0$.
\item Second, we use the first part to show that (\ref{eqn:single}) holds.
\end{enumerate}

\paragraph{Step 1.}

We prove that $\|h_y-h_y^*\|_{\infty}\xrightarrow{p}\to0$; the claim $\|g-g^*\|_{\infty}\xrightarrow{p}0$ follows similarly. First, note that
\begin{align*}
h_y^*(t)&=\int\mathbb{I}[f(x)=y]\cdot\mathbb{I}[x\xrightarrow{N^*}T^*]\cdot\mathbb{I}[x\le t]\cdot p(x)dx \\
&=\int\mathbb{I}[f(x)=y]\cdot\mathbb{I}[x\le t]\cdot p_{N^*}(x)dx,
\end{align*}
and define
\begin{align*}
\tilde{h}_y(t)=\int\mathbb{I}[f(x)=y]\cdot\mathbb{I}[x\le t]\cdot p_N(x)dx.
\end{align*}
Then, note that
\begin{align*}
\|h_y-h_y^*\|_{\infty}\le\|h_y-\tilde{h}_y\|_{\infty}+\|\tilde{h}_y-h_y^*\|_{\infty}.
\end{align*}
Bounding the first term, which represents the estimation error, is somewhat involved, so we relegate the proof to another lemma. In particular, taking $g=h_y$ and $g^*=\tilde{h}_y$ in Lemma~\ref{lem:estimation}, it follows that $\|h_y-\tilde{h}_y\|_{\infty}\xrightarrow{p}0$.

To bound the second term, note that
\begin{align*}
\|\tilde{h}_y-h_y^*\|_{\infty}&=\sup_{t\in\mathbb{R}}\left|\int\mathbb{I}[f(x)=y]\cdot\mathbb{I}[x\le t]\cdot(p_N(x)-p_{N^*}(x))dx\right| \\
&\le\sup_{t\in\mathbb{R}}\int|p_N(x)-p_{N^*}(x)|dx \\
&=\|p_N-p_{N^*}\|_1 \\
&\xrightarrow{p}0,
\end{align*}
where the last step follows by our assumption.

\paragraph{Step 2.}

Let $\epsilon,\delta>0$ be arbitrary. We need to show that
\begin{align*}
\left|\frac{h_y^*(t)^2}{g^*(t)}-\frac{h_y(t)^2}{g(t)}\right|\le\epsilon
\end{align*}
for every $t\in\mathbb{R}$ with probability at least $1-\delta$. By the previous step, we can take
\begin{align*}
\|g-g^*\|_{\infty}&\le\frac{\epsilon}{8} \\
\|h_y-h_y^*\|_{\infty}&\le\frac{\epsilon^2}{16}
\end{align*}
each with probability at least $1-\frac{\delta}{2}$, so by a union bound, both these inequalities hold with probability at least $1-\delta$.

We consider two cases. First, suppose that $g^*(t)\le\frac{\epsilon}{4}$, in which case
\begin{align*}
g(t)\le g^*(t)+\frac{\epsilon}{8}\le\frac{\epsilon}{2}.
\end{align*}
Then, since $h_y^*(t)\le g^*(t)$ and $h_y(t)\le g(t)$, we have
\begin{align*}
\left|\frac{h_y^*(t)^2}{g^*(t)}-\frac{h_y(t)^2}{g(t)}\right|&\le\left|\frac{h_y^*(t)^2}{g^*(t)}\right|+\left|\frac{h_y(t)^2}{g(t)}\right| \\
&\le|g^*(t)|+|g(t)| \\
&\le\epsilon.
\end{align*}
One detail is that when $g^*(t)=0$, then $H^*(t)$ is not well-defined. Defining $H^*(t)=0$ in this case is standard practice, since $h_y^*(t)\le g^*(t)$, so
\begin{align*}
H^*(t)=\frac{h_y^*(t)^2}{g^*(t)}\le\frac{g^*(t)^2}{g^*(t)}\le g^*(t)=0.
\end{align*}
Similarly, we define $H(t)=0$ if $g(t)=0$. In either case, the above argument still applies.

Second, suppose that $g^*(t)\ge\frac{\epsilon}{4}$, in which case
\begin{align*}
g(t)\ge g^*(t)-\frac{\epsilon}{8}\ge\frac{\epsilon}{8}.
\end{align*}
Then, we have
\begin{align*}
\left|\frac{h_y^*(t)^2}{g^*(t)}-\frac{h_y(t)^2}{g(t)}\right|&\le\frac{8}{\epsilon}\cdot|h_y^*(t)^2-h_y(t)^2| \\
&=\frac{8}{\epsilon}\cdot|h_y^*(t)-h_y(t)|\cdot|h_y^*(t)+h_y(t)| \\
&\le\frac{8}{\epsilon}\cdot\frac{\epsilon^2}{16}\cdot2 \\
&\le\epsilon.
\end{align*}

In either case, the claim follows, completing the proof.

\end{proof}

Next, we prove that the estimation error in Lemma~\ref{lem:gini} goes to zero.

\begin{lemma}
\label{lem:estimation}
\rm
Let $\mathcal{P}$ be a probability distribution over $\mathbb{R}$, let $p(x)$ be the probability density function for $\mathcal{P}$, let $F(x)$ be the cumulative distribution function for $\mathcal{P}$, let $\alpha:\mathbb{R}\to[0,1]$ be an arbitrary function, let
\begin{align*}
g^*(t)=\int\alpha(x)\cdot\mathbb{I}[x\le t]\cdot p(x)dx,
\end{align*}
and let $x^{(1)},...,x^{(n)}$ be i.i.d. random samples from $\mathcal{P}$, and let
\begin{align*}
g(t)=\frac{1}{n}\sum_{j=1}^n\alpha(x^{(j)})\cdot\mathbb{I}[x^{(j)}\le t]
\end{align*}
be the empirical estimate of $g^*$ on these samples. Then, we have
\begin{align*}
\text{Pr}_{x^{(1)},...,x^{(n)}\sim\mathcal{P}}\left[\|g-g^*\|_{\infty}\ge\frac{4\log n}{\sqrt{n}}\right]\le\frac{2}{n^{3/2}},
\end{align*}
for sufficiently large $n$.
\end{lemma}
\begin{proof}
First, we define points $t_0,t_1,...,t_{\sqrt{n}}\in\mathbb{R}$ that divide $\mathbb{R}$ into $\sqrt{n}$ intervals according to the cumulative distribution function $F(x)$ (for convenience, we assume $n$ is a perfect square). In particular, we choose $t_i$ to satisfy
\begin{align*}
t_i\in F^{-1}\left(\frac{i}{\sqrt{n}}\right).
\end{align*}
For convenience, we choose $t_0=-\infty$ and $t_{\sqrt{n}}=\infty$, which satisfy the condition. Now, for each $i\in[\sqrt{n}]$, let $I_i=(t_{i-1},t_i]$. Note that these intervals cover $\mathbb{R}$, i.e., $\mathbb{R}=I_1\cup...\cup I_{\sqrt{n}}$.

Then, we can decompose the quantity $\|g-g^*\|_{\infty}$ into three parts:
\begin{align*}
\|g-g^*\|_{\infty}&=\sup_{t\in\mathbb{R}}|g(t)-g^*(t)| \\
&=\sup_{i\in[\sqrt{n}]}\sup_{t\in I_i}|g(t)-g^*(t)| \\
&\le\sup_{i\in[\sqrt{n}]}\sup_{t\in I_i}\left\{|g(t)-g(t_i)|+|g(t_i)-g^*(t_i)|+|g^*(t_i)-g^*(t)|\right\} \\
&\le\sup_{i\in[\sqrt{n}]}\sup_{t\in I_i}|g(t)-g(t_i)|+\sup_{i\in[\sqrt{n}]}|g(t_i)-g^*(t_i)|+\sup_{i\in[\sqrt{n}]}\sup_{t\in I_i}|g^*(t_i)-g^*(t)|.
\end{align*}
We show that each of these three parts can be made arbitrarily small with high probability by taking $n$ sufficiently large.

\paragraph{First term.}

We first show that for every $i\in[\sqrt{n}]$, the interval $I_i$ contains at most $n^{1/2}\log n$ of the points $x^{(1)},...,x^{(n)}$ with high probability. By the definition of the points $t_i$, the probability that a single randomly selected point $x^{(j)}$ falls in $I_i$ is $n^{-1/2}$ (since the points $t_i$ were constructed according to the cumulative distribution function $F$):
\begin{align*}
M=\mathbb{E}_{x\sim\mathcal{P}}\left[\mathbb{I}[x\in I_i]\right]=\text{Pr}_{x\sim\mathcal{P}}[x\in I_i]=\frac{1}{\sqrt{n}}.
\end{align*}
Then, the fraction of the $n$ points $x^{(j)}$ that fall in the interval $I_i$ is
\begin{align*}
\hat{M}=\frac{1}{n}\sum_{j=1}^n\mathbb{I}[x^{(j)}\in I_i].
\end{align*}
Note that each $\mathbb{I}[x^{(j)}\in I_i]$ is an random variable in $[0,1]$, so by Hoeffding's inequality, we have
\begin{align*}
\text{Pr}_{x^{(1)},...,x^{(n)}\sim\mathcal{P}}\left[|\hat{M}-M|\ge\frac{2\log n}{\sqrt{n}}\right]\le2e^{-2(\log n)^2}\le\frac{1}{n^2}
\end{align*}
for sufficiently large $n$. Now, note that each point $x^{(j)}$ in $I_i$ can increase the value of $|g(t)-g(t_i)|$ by at most $n^{-1}$. Since there are $n\cdot\hat{M}$ points $x^{(j)}$ in $I_i$, the total increase is bounded by $\hat{M}$, i.e.,
\begin{align*}
\text{Pr}_{x^{(1)},...,x^{(n)}\sim\mathcal{P}}\left[\sup_{t\in I_i}|g(t)-g(t_i)|\ge\frac{2\log n}{\sqrt{n}}\right]\le\frac{1}{n^2}.
\end{align*}
By a union bound, this inequality holds for every $i\in[\sqrt{n}]$ with probability $n^{-3/2}$.

\paragraph{Second term.}

Note that each $\alpha(x^{(j)})\cdot\mathbb{I}[x^{(j)}\ge t]$ is a random variable in $[0,1]$. Therefore, by the Hoeffding inequality, we have
\begin{align*}
\text{Pr}_{x^{(1)},...,x^{(n)}\sim\mathcal{P}}\left[|g(t_i)-g^*(t_i)|\ge\frac{\log n}{\sqrt{n}}\right]\le2e^{-2(\log n)^2}\le\frac{1}{n^2}
\end{align*}
for sufficiently large $n$. By a union bound, this inequality holds for every $i\in[\sqrt{n}]$ with probability $n^{-3/2}$.

\paragraph{Third term.}

Since $t\le t_i$, we have $\mathbb{I}[x\le t]\le\mathbb{I}[x\le t_i]$. Thus, for all $t\in I_i$, we have
\begin{align*}
|g^*(t_i)-g^*(t)|
&=\left|\int\alpha(x)\cdot(\mathbb{I}[x\le t_i]-\mathbb{I}[x\le t])\cdot p(x)dx\right| \\
&\le\int(\mathbb{I}[x\le t_i]-\mathbb{I}[x\le t])\cdot p(x)dx \\
&=F(t_i)-F(t) \\
&\le n^{-1/2},
\end{align*}
where the last inequality follows from the definition of $t_i$ and the fact that $t\in I_i$.

\paragraph{Combined bound.}

Putting the three results together, we can conclude that for sufficiently large $n$, we have
\begin{align*}
\text{Pr}_{x^{(1)},...,x^{(n)}\sim\mathcal{P}}\left[\|g-g^*\|_{\infty}\ge\frac{4\log n}{\sqrt{n}}\right]\le\frac{2}{n^{3/2}},
\end{align*}
as claimed.
\end{proof}

\subsection{Proof of Regularity of the Gain Function}

In this section, we prove that the gain function $G^*$ satisfies certain regularity conditions.

\begin{lemma}
\label{lem:ginigap}
The function $G^*:\mathbb{R}\to\mathbb{R}$ is continuous and has bounded support.
\end{lemma}
\begin{proof}
It is clear that $G^*$ is continuous. To see that $G^*$ has bounded support, recall that $p(x)$ has bounded support, i.e., $p(x)=0$ for $|x|>x_{\text{max}}$. Then, note that if $s>x_{\text{max}}$, we have
\begin{align*}
\text{Pr}_{x\sim\mathcal{P}}[C_{N^*}\wedge(x\le s)]&=\text{Pr}_{x\sim\mathcal{P}}[C_{N^*}] \\
\text{Pr}_{x\sim\mathcal{P}}[C_{N^*}\wedge(x>s)]&=0 \\
\text{Pr}_{x\sim\mathcal{P}}[f(x)=y\mid C_{N^*}\wedge(x\le s)]&=\text{Pr}_{x\sim\mathcal{P}}[f(x)=y\mid C_{N^*}] \\
\text{Pr}_{x\sim\mathcal{P}}[f(x)=y\mid C_{N^*}\wedge(x>s)]&=0
\end{align*}
Therefore, we have
\begin{align*}
G^*(s)=-H^*(f,C_{N^*}\wedge(x\le s))-H(f,C_{N^*}\wedge(x>s))+H(f,C_{N^*})=0.
\end{align*}
By a similar argument, $G^*(s)=0$ for $s<-x_{\text{max}}$, so the claim follows.
\end{proof}

\section{User Study Interface}
\label{sec:userstudyappendix}

We include images of our user study on the following pages. Where applicable, the marked answers are the correct ones. Within each part, we randomized the order in which the decision tree and the rule list or decision set appeared.

\clearpage

\includegraphics[width=\textwidth]{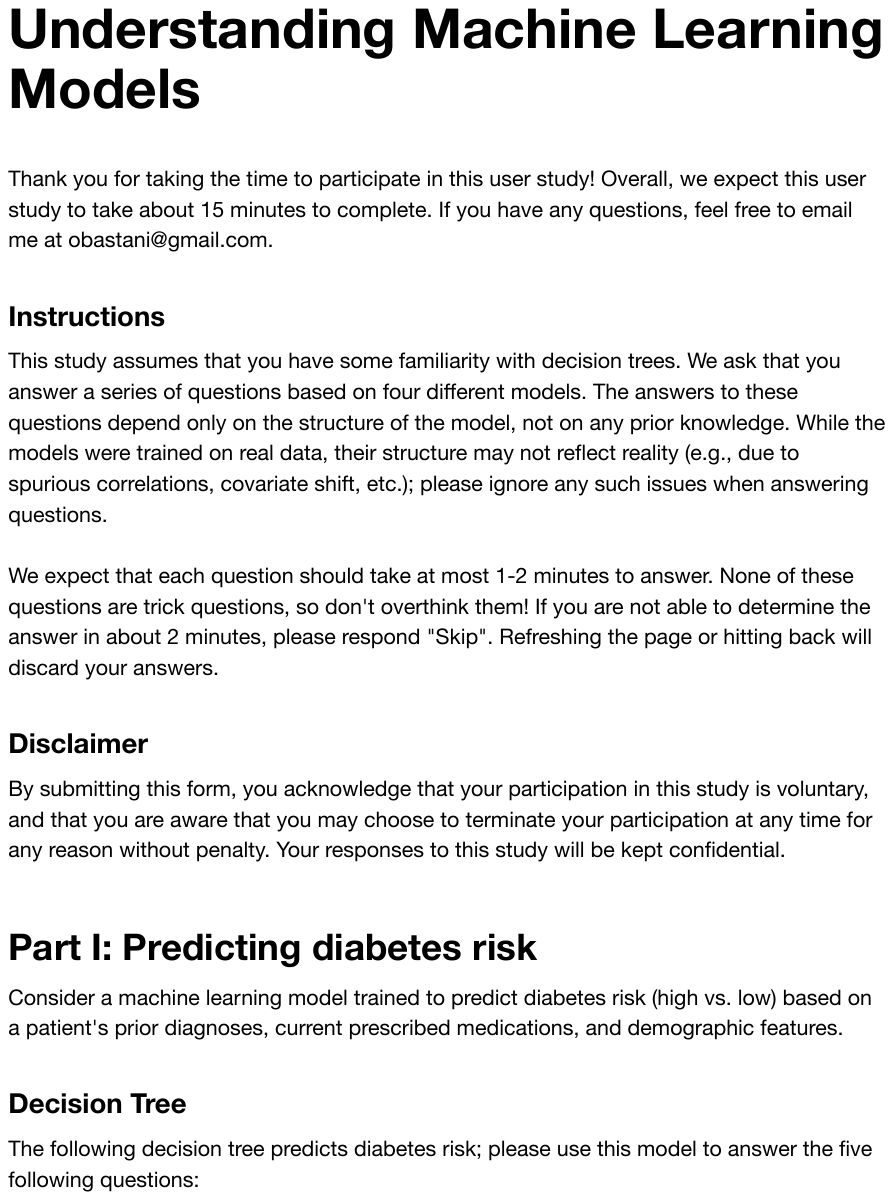} \\

\includegraphics[width=\textwidth]{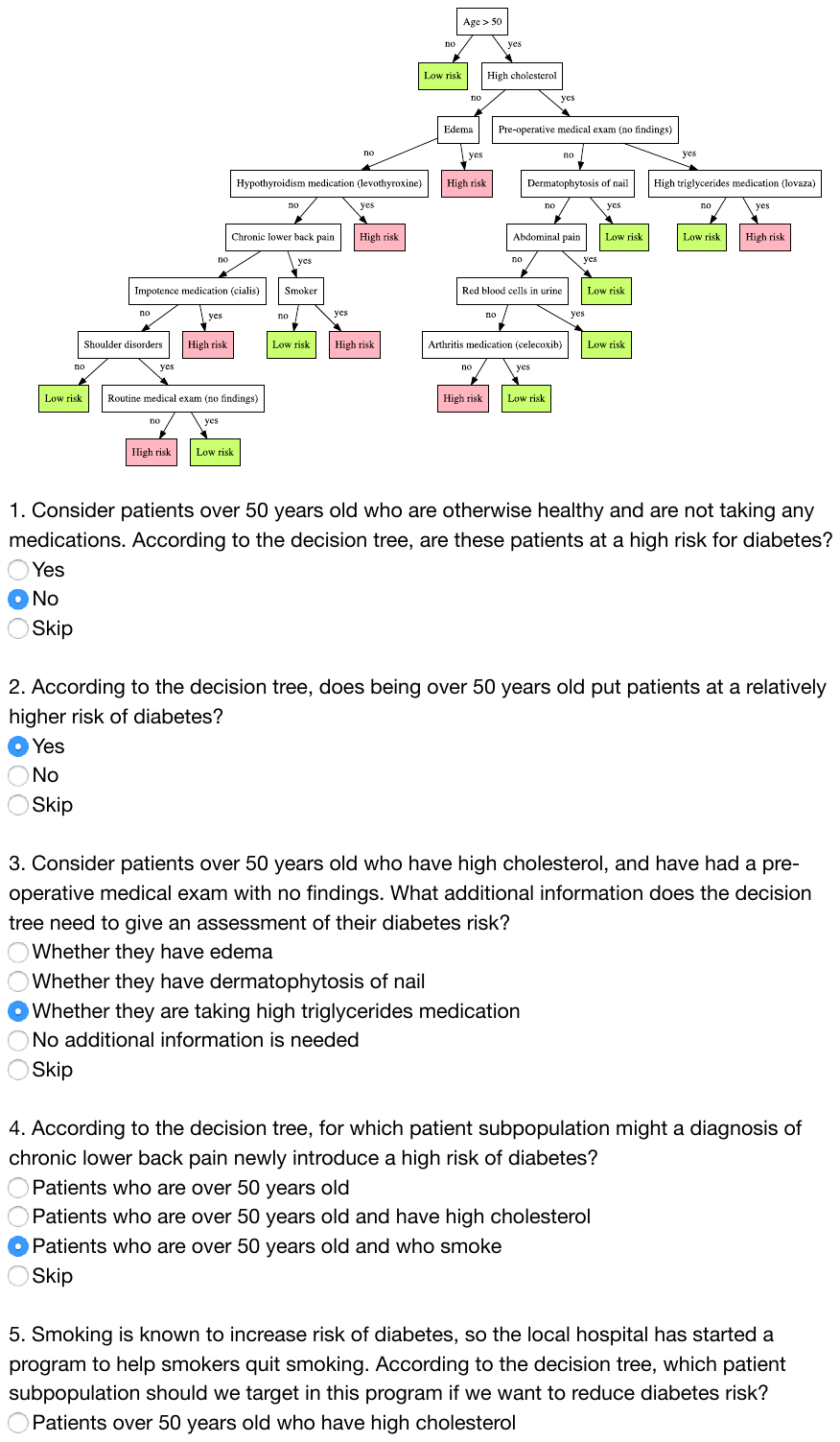} \\

\includegraphics[width=\textwidth]{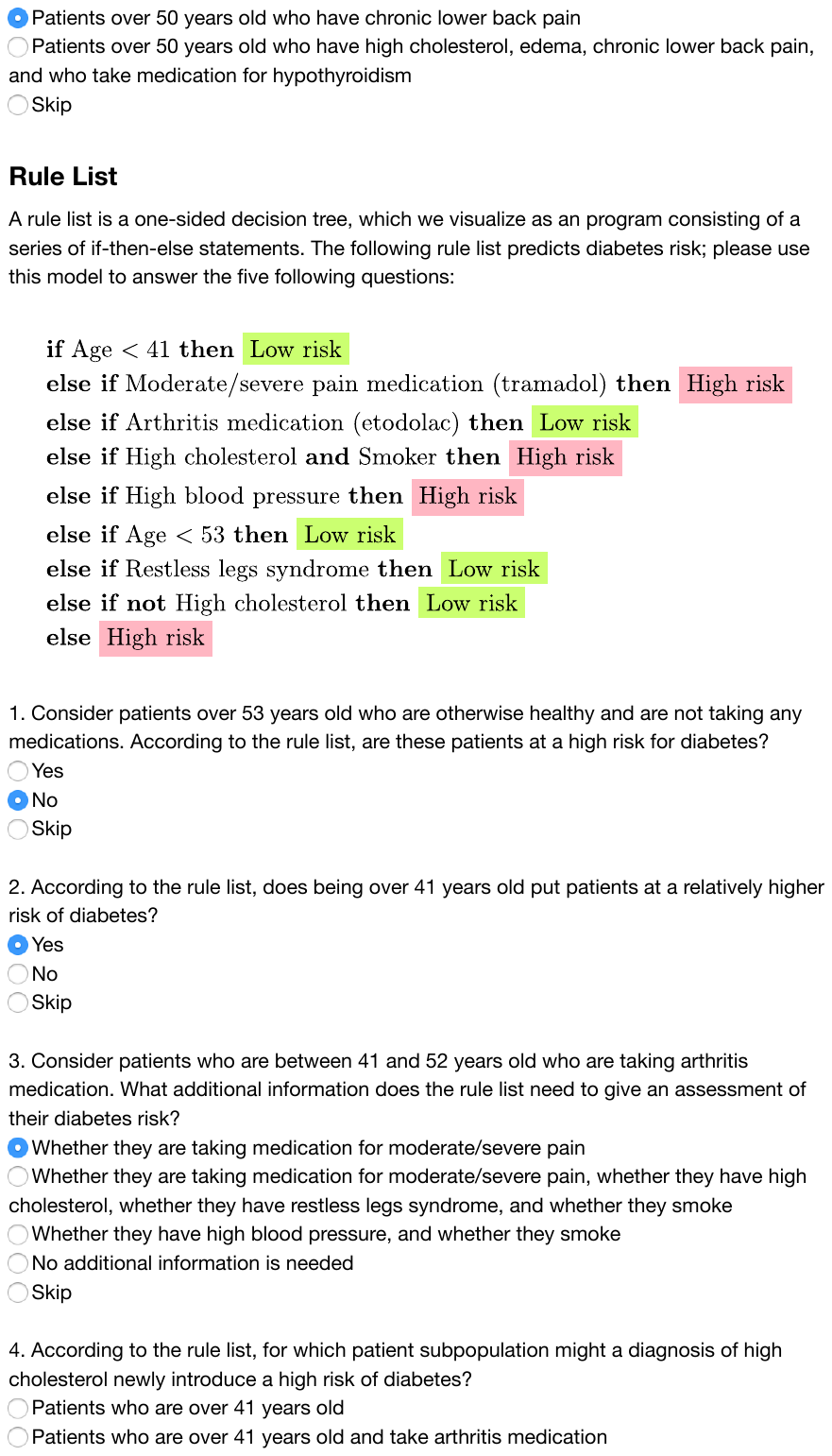} \\

\includegraphics[width=\textwidth]{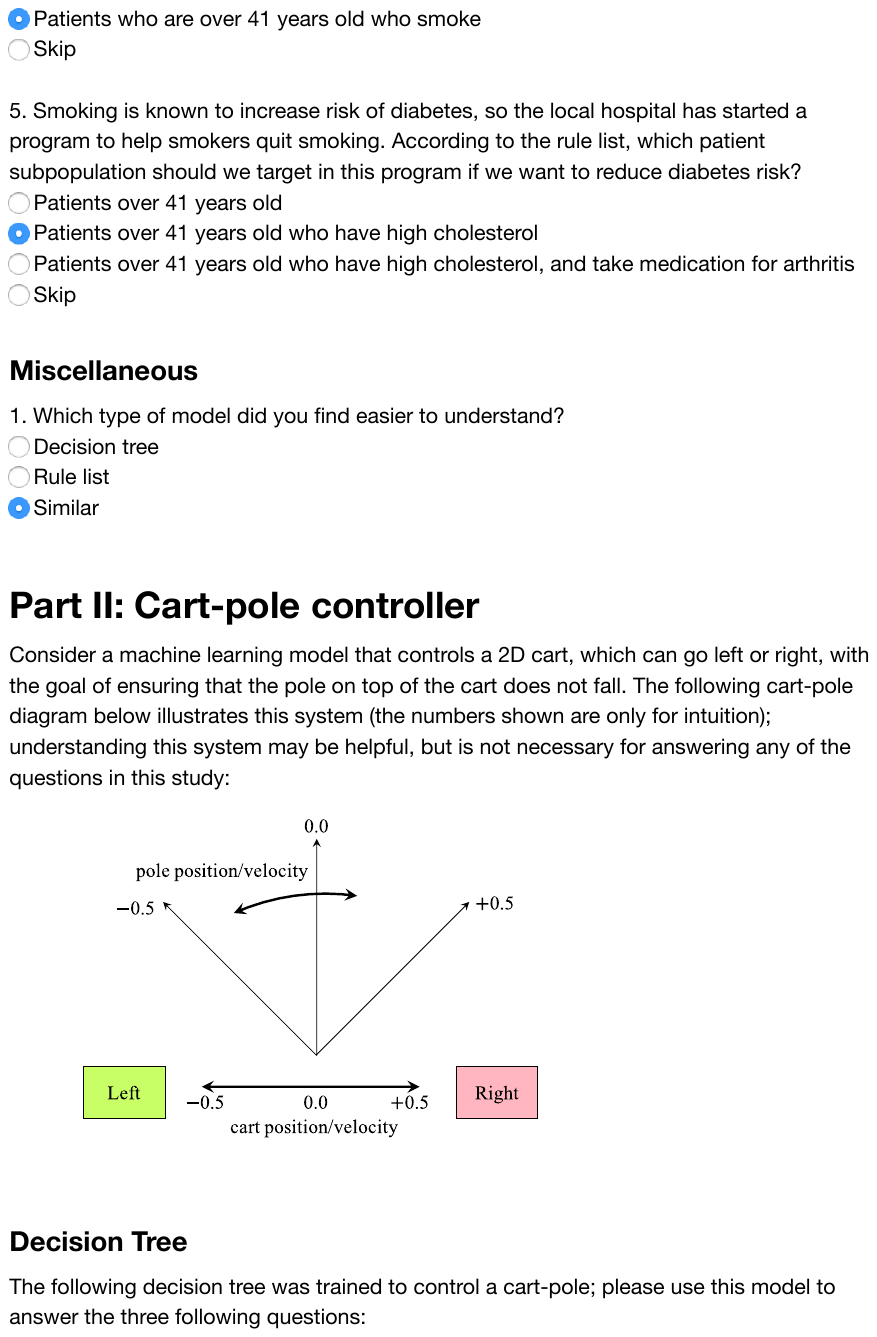} \\

\includegraphics[width=\textwidth]{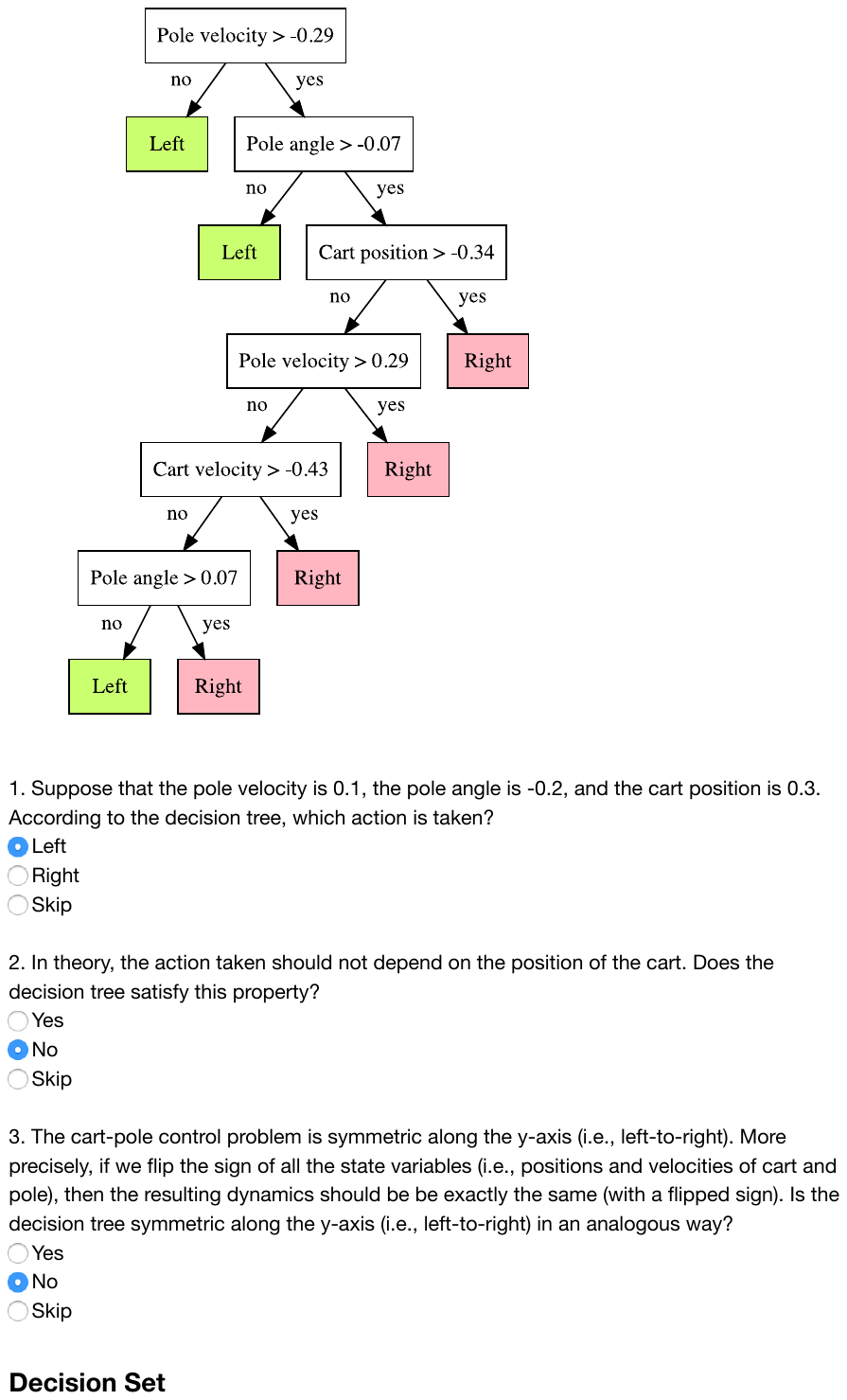} \\

\includegraphics[width=\textwidth]{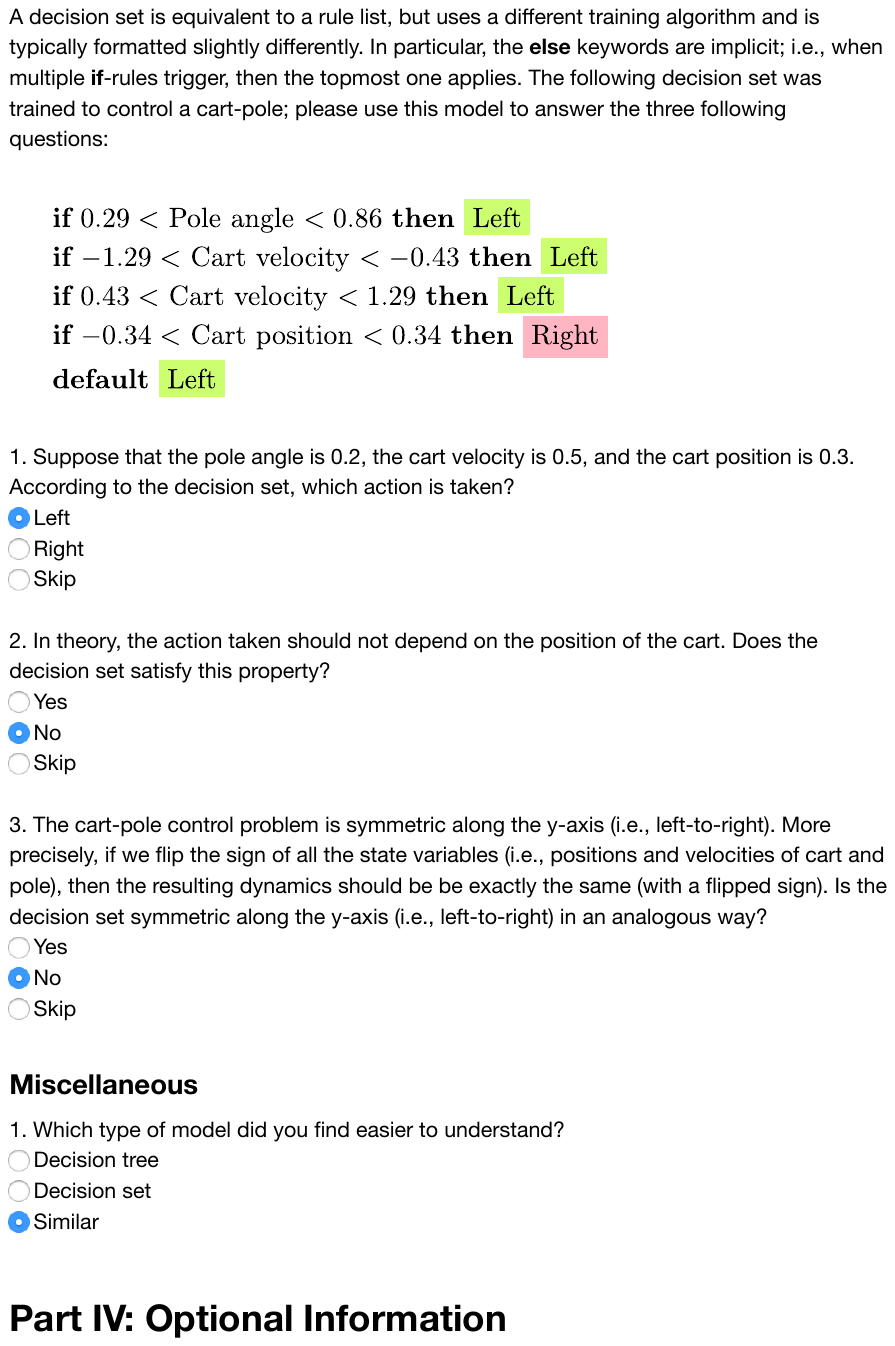} \\

\includegraphics[width=\textwidth]{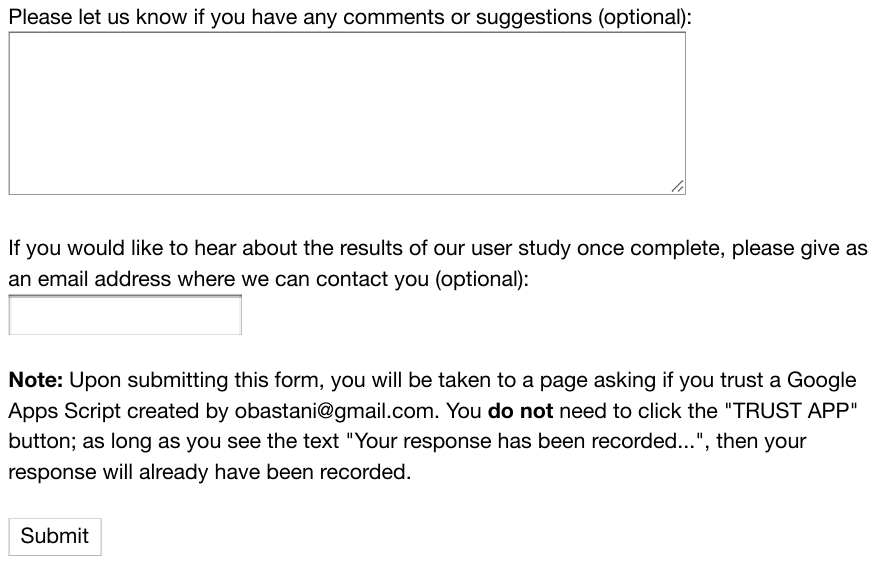} \\

\end{document}